\documentclass[twoside]{article}

\usepackage[accepted]{styles/aistats2024}
\runningauthor{D.~Robert-Nicoud, A.~Krause, V.~Borovitskiy}

\usepackage{hyperref}
\usepackage{cleveref}
\hypersetup{
    colorlinks,
    linkcolor={red!50!black},
    citecolor={blue!50!black},
    urlcolor={blue!80!black}
}

\usepackage{styles/VABcommands}[center]
\usepackage{styles/environments}
\usepackage[authoryearcomp]{styles/VABcitations}

\usepackage{url}
\usepackage{subcaption}
\usepackage{booktabs}
\usepackage{multirow}
\usepackage{tikz}
\usepackage{standalone}
\usepackage{lipsum}
\usepackage{csquotes}
\MakeOuterQuote{"}
\usepackage{enumitem}
\usepackage{booktabs}
\usepackage{thmtools} 
\usepackage{thm-restate}

\bibliography{references.bib}

\usepackage{color}
\usepackage{todonotes}
\usepackage{amsmath}
\usepackage{amssymb}
\usepackage{afterpage}

\newtheorem{definition}{Definition}
\newtheorem{theorem}[definition]{Theorem}
\newtheorem{proposition}[definition]{Proposition}

\newtheorem{lemma}[definition]{Lemma}
\newtheorem{corollary}[definition]{Corollary}
\newtheorem{remark}[definition]{Remark}

\crefname{definition}{Definition}{Definitions}
\crefname{theorem}{Theorem}{Theorems}
\crefname{proposition}{Proposition}{Propositions}
\crefname{result}{Result}{Results}
\crefname{lemma}{Lemma}{Lemmas}
\crefname{corollary}{Corollary}{Corollaries}
\crefname{remark}{Remark}{Remarks}
\crefname{example}{Example}{Examples}

\newcommand{\Kxx}[2]{\ensuremath{\mathrm{K}_{#1#2}}}
\newcommand{\M}{\ensuremath{\mathrm{M}}}
\newcommand{\NN}{\ensuremath{\mathrm{N}}}

\newcommand{\tk}{\ensuremath{\v{k}}}
\newcommand{\im}{\ensuremath{\operatorname{im}}}

\newcommand{\LTM}{\ensuremath{\c{L}^2(\M;T\M)}}

\newcommand{\ssf}{\mathrm{I\!I}}

\begin{document}

\pagenumbering{arabic}
\twocolumn[

\aistatstitle{Intrinsic Gaussian Vector Fields on Manifolds}

\aistatsauthor{ Daniel Robert-Nicoud \And Andreas Krause \And Viacheslav Borovitskiy\vspace*{3ex}}

\aistatsaddress{ETH Z\"urich, Switzerland}]

\begin{abstract}
Various applications ranging from robotics to climate science require modeling signals on non-Euclidean domains, such as the sphere.
Gaussian process models on manifolds have recently been proposed for such tasks, in particular when uncertainty quantification is needed.
In the manifold setting, vector-valued signals can behave very differently from scalar-valued ones, with much of the progress so far focused on modeling the latter.
The former, however, are crucial for many applications, such as modeling wind speeds or force fields of unknown dynamical systems.
In this paper, we propose novel Gaussian process models for vector-valued signals on manifolds that are intrinsically defined and account for the geometry of the space in consideration.
We provide computational primitives needed to deploy the resulting \emph{Hodge--Matérn Gaussian vector fields} on the two-dimensional sphere and the hypertori.
Further, we highlight two generalization directions: discrete two-dimensional meshes and "ideal" manifolds like hyperspheres, Lie groups, and homogeneous spaces.
Finally, we show that our Gaussian vector fields constitute considerably more refined inductive biases than the extrinsic fields proposed before.
\end{abstract}

\makeatletter
\fancyhead[CE]{\small\bfseries\@runningtitle}
\fancyhead[CO]{\small\bfseries\@runningauthor}
\makeatother

\section{\bfseries\small INTRODUCTION}

Gaussian processes \cite{rasmussen2006} are a widely used class of Bayesian models in machine learning.
They are known to perform well in small data scenarios and to provide well-calibrated uncertainty.
Their notable applications include modeling spatial data \cite{chiles2012} and automated decision-making, e.g., optimization \cite{snoek2012,shields2021,noskova2023}, or sensor placement \cite{krause2008}.

Gaussian processes can be scalar- or vector-valued~\cite{alvarez2012}.
The important special case of the latter is \emph{Gaussian vector fields}.
These can for example be used to model velocities or accelerations, either as a target in itself or as means for exploring an unknown dynamical system.
When the input domain is Euclidean, vector fields are just vector-valued functions.
However, when the domain is a submanifold of a Euclidean space, such as when modeling wind speeds or ocean currents on the surface of Earth, the situation can be quite different: geometry places additional constraints on vector fields that need to be accounted for.
As illustrated in~\Cref{fig:fields_vs_functions}, the values of a \emph{vector field} ought to be tangential to the manifold, while those of a vector function can be arbitrary vectors.

\begin{figure}[b]
    \centering
    \begin{subfigure}[b]{0.45\linewidth}
        \includegraphics[width=\linewidth]{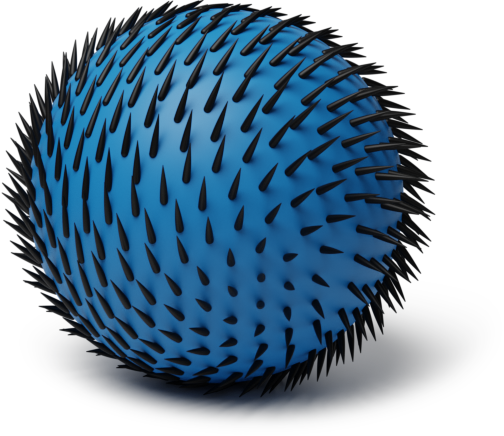}
        \caption{Vector Function}
    \end{subfigure}
    \hspace{0.01\linewidth}
    \begin{subfigure}[b]{0.45\linewidth}
        \includegraphics[width=\linewidth]{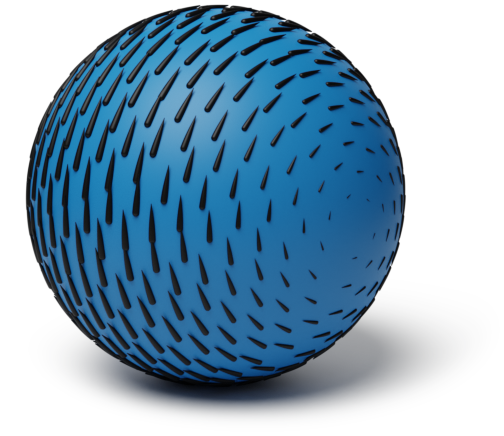}
        \caption{\label{fig:rotation vf}Vector Field}
    \end{subfigure}
\caption{Vector fields on manifolds are not just vector functions, for them the vectors are always \emph{tangential}.}
\label{fig:fields_vs_functions}
\end{figure}

\begin{table}[b!]
\footnotesize
Code: \url{https://github.com/DanielRobertNicoud/imv-gps}.\\
Also, see \url{https://github.com/GPflow/GeometricKernels}.
\\*
Correspondence to: \href{mailto:daniel.robertnicoud@gmail.com}{daniel.robertnicoud@gmail.com} and \href{mailto:viacheslav.borovitskiy@gmail.com}{viacheslav.borovitskiy@gmail.com}.
\end{table}

In recent years, two different formalisms were proposed for defining Gaussian vector fields on manifolds. \textcite{lange2018algorithmic} approached the problem by considering linear constraints on vector-valued Gaussian processes, constraining them to lie in  the tangent space of a submanifold of $\R^d$. Additionally, one can impose further linear constraints to the resulting fields, such as making them divergence-free. \textcite{hutchinson2021}---which is closer in spirit to this work---instead considered projecting vector-valued Gaussian processes to a submanifold of $\R^d$. While both these procedures can in principle produce any valid Gaussian vector field, they are fundamentally \emph{extrinsic} and we will show that the fields one gets in practice introduce undesirable inductive biases.

To remedy this, we propose a new approach: fully intrinsic Gaussian vector fields based on the \emph{Hodge Laplacian}\footnote{The concurrent work by \textcite{peach2024} studies Gaussian vector fields induced by the \emph{connection Laplacian} rather than the Hodge Laplacian (see~\Cref{subsection:Bochner Laplacian} on the difference between these two notions of Laplacian). \mbox{Importantly}, they consider a very different setting: a priori unknown manifolds that are estimated from finite data, showcasing an impressive neuroscience application.} that we name \emph{Hodge--Matérn Gaussian vector fields}.
For some simple manifolds, namely for the two-dimensional sphere $\mathbb{S}_2$ and for the hypertori $\mathbb{T}^d$, we develop computational techniques that allow effortless use of these intrinsic fields in downstream applications.

The aforementioned computational techniques hinge on knowing the \emph{eigenvalues} and \emph{eigenfields} of the Hodge Laplacian.
To this end, we describe how to:
\begin{enumerate}[label={\alph*)}]
\item derive these from the eigenpairs of the Laplace--Beltrami operator when the manifold is two-dimensional, using automatic differentiation only;
\item compute these on product manifolds in terms of the eigenvalues and eigenfields on the factors; and
\item get these explicitly in the cases of the circle~$\mathbb{S}_1$, the hypertori~$\mathbb{T}^d$, and the sphere~$\mathbb{S}_2$.
\end{enumerate}
We \emph{conjecture} that (a) can also be used to define Gaussian vector fields on meshes, by changing the analytic notions into their appropriate discretizations.
Furthermore, we \emph{conjecture} that (c) can be done for more general \emph{parallelizable} manifolds, e.g.~on \emph{Lie groups}, which may then facilitate the generalization to the very general class of \emph{homogeneous spaces}.
The latter includes many manifolds of interest which are poorly amenable to discretization because of their higher dimension.

By showing how our intrinsic Hodge--Matérn Gaussian vector fields improve over their naïve extrinsic counterparts on the two-dimensional sphere, we hope to motivate further research.
First, into the development of practical intrinsic Gaussian vector fields on other domains.
Second, into applying the proposed models in areas like climate/weather modeling and robotics.

\subsection{Gaussian Processes}

Let $X$ be a set. A random function $f$ on $X$ is called a \emph{Gaussian process} (GP) with mean $\mu:X\to\R$ and covariance (or kernel) $k:X\times X\to\R$---denoted by $f \sim \f{GP}(\mu, k)$---if for any finite set of points $\v{x}$ in $X$ we have $f(\v{x}) \sim \f{N}(\mu(\v{x}), \Kxx{\v{x}}{\v{x}})$ where $\Kxx{\bdot}{\,\bdot'}=k(\bdot,\bdot')$.
Without loss of generality, we usually assume $\mu(\cdot) = 0$.

Assuming a GP prior $f\sim\GP(0,k)$ and a Gaussian likelihood
$
\v{y} \given f = \mathcal{N}(\v{y}\given f(\v{x}),\sigma_{\eps}^2)
$
with a fixed noise variance $\sigma_{\eps}^2$, the posterior $f \given \v{y}$ is a GP \cite{rasmussen2006} with mean and covariance
\begin{align}
    \mu_{f \given \v{y}}(\cdot)={}&\Kxx{\cdot}{\v{x}}\left(\Kxx{\v{x}}{\v{x}}+\sigma_{\eps}^2\mathbf{I}\right)^{-1}\v{y},
    \\
    k_{f \given \v{y}}(\cdot,\cdot')={}&k(\cdot,\cdot') - \Kxx{\cdot}{\v{x}}\left(\Kxx{\v{x}}{\v{x}}+\sigma_{\eps}^2\mathbf{I}\right)^{-1}\Kxx{\v{x}}{\cdot'}.
\end{align}
Here, the function $\mu_{f \given \v{y}}$ can be used to draw \emph{predictions} and the function $k_{f \given \v{y}}$ is used to quantify \emph{uncertainty}.

\looseness -1
When $X=\R^d$, \emph{Matérn Gaussian processes} \cite{rasmussen2006, stein1999} are most often used. Their respective kernels $k_{\nu, \kappa, \sigma^2}$ are the three-parameter family of \emph{Matérn kernels}, whose limiting case $k_{\infty, \kappa, \sigma^2}$ for $\nu \to \infty$ is known as the heat (a.k.a.~squared exponential, RBF, Gaussian, diffusion) kernel, which is arguably the most popular.

\subsection{Gaussian Processes on Manifolds}\label{sec:scalar GP on manifolds}
Now consider $X = \M$, where $\M$ is a compact Riemannian manifold. Throughout this paper, manifolds are always assumed to be connected.
Using the SPDE-based characterization of Matérn Gaussian processes of~\textcite{whittle1963,lindgren2011},
\textcite{borovitskiy2020} showed how to compute Matérn kernels on $\M$ in terms of the spectrum of the Laplace--Beltrami operator $\Delta$:
\[ \label{eqn:matern_on_manifolds}
k_{\nu, \kappa, \sigma^2}(x, x')
=
\frac{\sigma^2}{C_{\nu, \kappa}}
\sum_{n=0}^{\infty}
\Phi_{\nu, \kappa}(\lambda_n)
f_n(x) f_n(x'),
\]
where $\cbr{f_n}_{n=0}^{\infty}$ is an orthonormal basis of eigenfunctions of $\Delta$ such that $\Delta f_n = -\lambda_n f_n$, 
\[ \label{eqn:phi_dfn}
\Phi_{\nu, \kappa}(\lambda)
=
\begin{cases}
\del{\frac{2 \nu}{\kappa^2} + \lambda}^{-\nu - d/2} & \nu < \infty, \\
e^{- \frac{\kappa^2}{2} \lambda} & \nu = \infty,
\end{cases}
\]
$d = \dim(\M)$, and $C_{\nu, \kappa}$ is a normalization constant that ensures $\tfrac{1}{\operatorname{vol}\M}\int_\M k_{\nu, \kappa, \sigma^2}(x, x) \d x = \sigma^2$. 
There exist analytical \cite{azangulov2022, azangulov2023} and numerical \cite{borovitskiy2020, coveney2020} techniques for computing the eigenpairs $\lambda_n, f_n$, or bypassing the computation thereof.
In the end, a truncated series from~\Cref{eqn:matern_on_manifolds} yields tractable Gaussian processes that respect the intrinsic geometry of the manifold \cite{rosa2023}, as illustrated in~\Cref{fig:intrinsic_vs_extrinsic_scalar}.

\begin{figure}[t]
\centering
\begin{subfigure}[b]{\linewidth}
    \includegraphics[width=\linewidth]{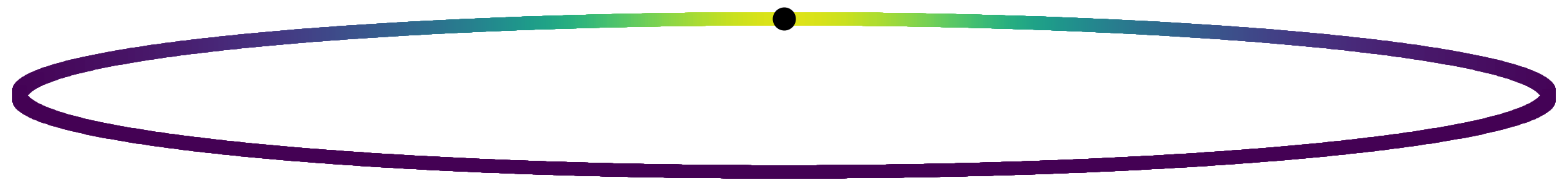}
    \caption{Intrinsic kernel}
\end{subfigure}

\vspace{1ex}

\begin{subfigure}[b]{\linewidth}
    \includegraphics[width=\linewidth]{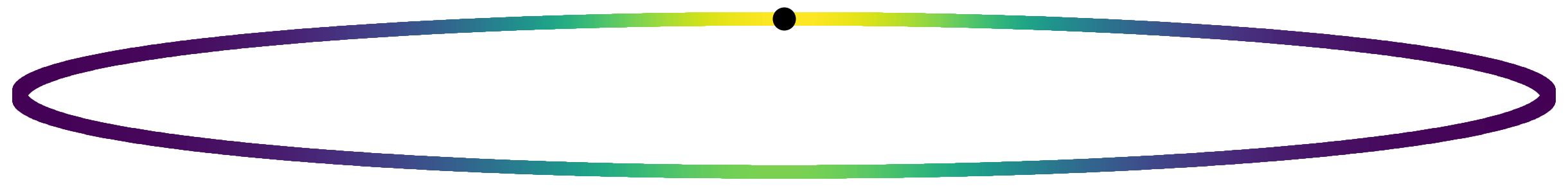}    
    \caption{Extrinsic kernel}
\end{subfigure}
\caption{Comparing an intrinsic Matérn kernel ($\nu=\infty$) of~\Cref{eqn:matern_on_manifolds} to an \emph{extrinsic} one, the restriction of a Euclidean Matérn kernel to the manifold. Note the latter induces high correlation between the points across the minor axis of the ellipse, despite them being far from each other in terms of the intrinsic distance.}
\label{fig:intrinsic_vs_extrinsic_scalar}
\end{figure}

\begin{figure*}[t!]%
\begin{subfigure}{0.25\textwidth}%
\centering%
\includegraphics[width=\textwidth]{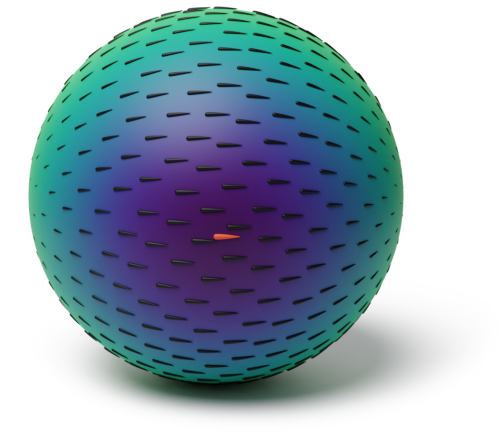}%
\caption{Projected Matérn}%
\end{subfigure}%
\begin{subfigure}{0.25\textwidth}%
\centering%
\includegraphics[width=\textwidth]{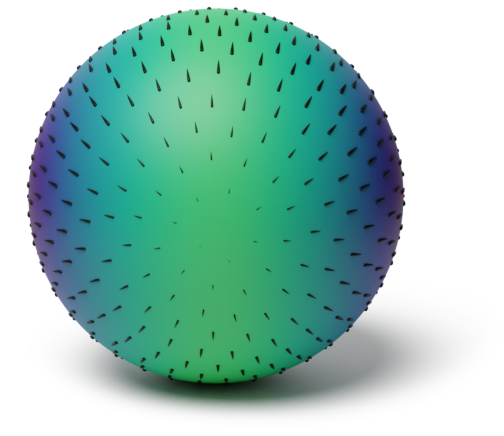}%
\caption{Projected Matérn, rotated}%
\end{subfigure}%
\begin{subfigure}{0.25\textwidth}%
\centering%
\includegraphics[width=\textwidth]{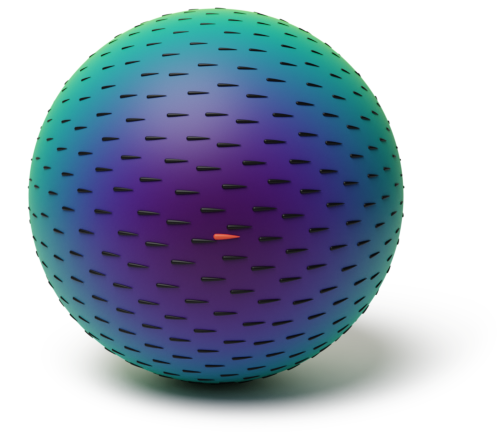}%
\caption{Hodge--Matérn}%
\end{subfigure}%
\begin{subfigure}{0.25\textwidth}%
\centering%
\includegraphics[width=\textwidth]{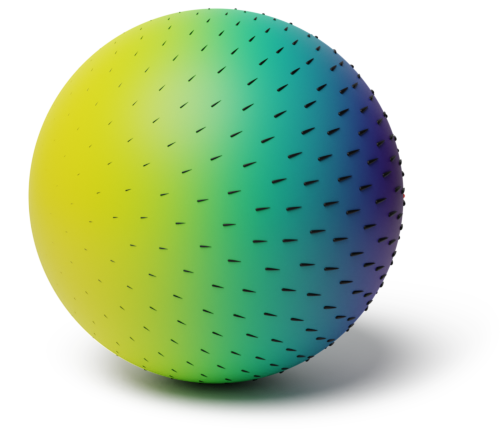}%
\caption{Hodge--Matérn, rotated}%
\end{subfigure}%
\caption{GP regression from a single observation (the red vector) for a very large length scale~$\kappa$. Black vectors represent the prediction $\mu_{f \given \v{y}}(\cdot)$, the background color shows the uncertainty $\norm{k_{f \given \v{y}}(\cdot,\cdot)}$: yellow for high, blue for low. On (a) and (b) we use a projected Matérn GP of \textcite{hutchinson2021}; on (c) and (d) we use our Hodge--Matérn Gaussian vector field, $\nu = \infty$ in both cases. Unnaturally, uncertainty in (a) and (b) is non-monotonous with respect to the distance on the sphere: it is considerably lower at the antipode than halfway~to~it.}%
\label{fig:projected_vs_hodge_gpr}%
\end{figure*}

Importantly, same as in the Euclidean case we have $k_{\infty, \kappa, \sigma^2}(x, x') \propto \c{P}(\tfrac{\kappa^2}{2}, x, x')$ \cite{azangulov2022}, where $\c{P}(t, x, x')$ is the \emph{heat kernel}: the solution of
\[
\frac{\partial \c{P}}{\partial t}
=
\Delta_{x} \c{P},
&&
\lim_{t \to 0} \c{P}(t, x, x') = \delta(x - x'),
\]
where $\Delta_{x}$ denotes the Laplace--Beltrami operator acting on the variable $x$ and $\delta$ is the Dirac delta function; convergence takes place in the sense of distributions.

\subsection{Gaussian Vector Fields on Manifolds} \label{sec:intro:vector_fields}

Following \textcite{hutchinson2021}, we introduce the notions of Gaussian vector fields and their kernels.

\begin{definition}
A \emph{random vector field} $f$ is a function mapping each $x \in \M$ to a random variable $f(x)$ with values in the \emph{tangent space} $T_{x} \M$ of $\M$ at $x$.
It is \emph{Gaussian} if $(f(x_1), \ldots, f(x_n)) \in T_{x_1} \M \oplus \ldots \oplus T_{x_n} \M$ is jointly Gaussian for all $x_1, \ldots, x_n \in \M$ and $n \in \N$.
\end{definition}

When $\M = \R^d$, this reduces to vector-valued GPs with output dimension equal to the input dimension $d$, e.g.~a concatenation of $d$ independent scalar-valued GPs.
If instead $\M\subseteq\R^d$ is an embedded manifold, concatenating $d$ scalar-valued GPs does not lead to a vector field: in this case, $f(x) \in \R^d$ rather than $f(x) \in T_{x} \M$ as illustrated by~\Cref{fig:fields_vs_functions}.

If $\v{g}$ is a vector-valued GP on $\R^d$, then its values $\v{g}(\v{x})$ are Gaussian vectors.
Thus, the kernel of $\v{g}$ is a matrix-valued function $\v{k}(\v{x}, \v{x}') = \Cov(\v{g}(\v{x}), \v{g}(\v{x}'))$.
Alternatively, it can be defined as a scalar-valued function $k((\v{x}, \v{u}), (\v{x}', \v{v})) = \Cov(\v{u}^{\top}\v{g}(\v{x}), \v{v}^{\top}\v{g}(\v{x}'))$ where $\v{x}, \v{x}', \v{u}, \v{v} \in \R^d$.
Reinterpreting the inner products $\v{a}^{\top} \v{b}$ as the linear functional $ \bullet \to \v{a}^{\top}  \bullet$ applied to the vector $\v{b}$ leads to a particularly elegant generalization.
For a Gaussian vector field $f$ on a manifold $\M$ we put
\[
    \tk(x, x')(u, v) ={}& k\big((x, u), (x', v)\big)\\
    ={}& \Cov\big(u(f(x)), v(f(x))\big),
\]
where $x, x' \in M$, $u \in T^{*}_{x}\M, v \in T^{*}_{x'}\M$ and $T^{*}_{x}\M$ denotes the \emph{cotangent} space of \emph{covectors}, linear functionals on~$T_{x}\M$.
Together with a deterministic \emph{mean vector field}~$\mu$, such a kernel $k$ determines the distribution of $f$ in a coordinate-free way, justifying the use of the standard notation ${f \sim \f{GP}(\mu, k)}$.

In order to define Gaussian vector fields that can be used in practice, \textcite{hutchinson2021} introduced the notion of \emph{projected Gaussian processes}.
These are constructed by picking an isometric embedding $\phi:\M\to\R^D$ into some Euclidean space.\footnote{Such an embedding always exists by the \emph{Nash embedding theorem}; it is not unique.}
Then, a tangent space $T_{x}\M$ can be identified with a subspace of $T_{\phi(x)}\R^D \cong \R^D$.
Thus, there exists a projection $\f{P}_{x}$ from $\R^D$ to this subspace and $f(x) = \f{P}_{x} \v{g}(x)$ defines a valid Gaussian vector~field for any vector-valued GP~$\v{g}$.

\section{\bfseries\small CHALLENGES}

The class of projected GPs we reviewed in~\Cref{sec:intro:vector_fields} is very large.
In fact, any Gaussian vector field $f$ can be represented as a projected GP, simply because $f = \f{P}_x \v{g}$ with $\v{g} = f$, as proved in~\textcite{hutchinson2021}.
However, in order to obtain $f$ using this trick we need to know it in the first place.

To avoid this "chicken and egg" problem, we should  construct $\v{g}$ using the tools we already posses.
Practically, this means stacking together scalar-valued GPs.
The first challenge is to find scalar GPs that respect the geometry of the manifold.
It can be solved by using the manifold Matérn GPs reviewed in  \Cref{sec:scalar GP on manifolds}.
To obtain an expressive family, we can define $\v{g} = \m{A} \v{h}$ where $h_j \sim \f{GP}(0, k_{\nu, \kappa_j, \sigma^2})$ are independent and $\m{A}$ is an arbitrary matrix.
In fact, this is exactly what \textcite{hutchinson2021} propose to do in practice.

However, as can be seen in \Cref{fig:projected_vs_hodge_gpr}, this construction can produce undesirable artifacts.
Upon a closer examination, the reasons for this can be formalized.
\begin{restatable}{proposition}{ThmLimitation} \label{thm:limitation}
With notation as above, if $\operatorname{rk}{\v{A}} > 1$, there are $x, x' \in \mathbb{S}_2$ so that $\angle (x, x') = 90^{\circ}$ but
\[
\!\!\!\norm{\Cov(f(x), f(x'))} \!<\! \norm{\Cov(f(x), f(\tilde{x}))},
&&
\!\kappa_j \!\to\! \infty,
\]
where $\tilde{x} = -x$ is antipodal to $x$ \textup{(}i.e.~$\angle (x, \tilde{x}) = 180^{\circ}$\textup{)} and $\norm{\bdot}$ denotes the Frobenius norm.
\end{restatable}

For a proof, see~\Cref{sect:limitations of proj kernels}.
\Cref{thm:limitation} implies that, for large length scales $\kappa$, the covariance of a projected GP is non-monotonic in the intrinsic distance on the sphere---this is an undesirable trait, usually.

Another issue is that the projected GP construction does not provide a way to force divergence-free or curl-free inductive biases.
These can be quite important for certain types of vector field data \cite{berlinghieri2023}, as we will clearly observe in~\Cref{sec:experiments}.

To overcome  these challenges, we introduce a fully intrinsic class of Gaussian vector fields on manifolds.

\section{\bfseries\small INTRINSIC GAUSSIAN VECTOR FIELDS}

In this section, we present the main ideas behind the
construction of the intrinsic Gaussian vector fields we propose.
The mathematical formalism for this section is detailed in~\Cref{appdx:theory}.

\subsection{The Hodge Heat Kernel}

We start by generalizing the heat kernel, i.e.~the Matérn kernel with $\nu=\infty$.

Let $\M$ be a compact, \emph{oriented} Riemannian manifold. Then, \emph{Hodge theory}---see e.g.~\textcite{rosenberg1997laplacian} for an approachable introduction---defines a generalization $\Delta$ of the Laplace--Beltrami operator that acts on vector fields on $\M$ instead of scalar functions, called the \emph{Hodge Laplacian}.
We consider the associated heat equation
\[ \label{eqn:vecor_heat}
\frac{\partial u}{\partial t}(t, x) = \Delta_x u(t, x),
\]
where $u$ is smooth in both variables and $u(t,x)\in T_x\M$.
This equation admits a \emph{fundamental solution}~$\v{\c{P}}$: the \emph{Hodge (heat) kernel} which for any choice of $t \in \R_{>0}$ and $x, x' \in M$ gives a function
\[
\v{\c{P}}(t, x, x'): T^*_{x}\M\otimes T^*_{x'}\M\longrightarrow\R.
\]
Considering $t$ as a hyperparameter, we obtain a function $\v{\c{P}}_t(x, x')$ with the exact signature that a kernel of Gaussian vector field should posses, as by~\Cref{sec:intro:vector_fields}.
We prove the following.

\begin{restatable}{theorem}{ThmExistenceGVFheatKernel}
For any $t > 0$ there exists a Gaussian vector field whose kernel is $\v{\c{P}}_t$. 
\end{restatable}

Having found a suitable adaptation of the heat kernel to the vector field case, we turn to making it explicit. 
Similarly to the scalar case, a Hilbert space $\LTM$ of square integrable vector fields can be defined.
There is always an orthonormal basis $\cbr{s_n}_{n=0}^{\infty}$ of $\LTM$ such that $\Delta s_n = - \lambda_n s_n$, i.e.~$s_n$ are eigenfields of the Hodge Laplacian.
Moreover, we have $\lambda_n \geq 0$ for $n \geq 0$.
The kernel $\v{\c{P}}_t$ can be computed in terms of the eigenfields just like its scalar counterpart (cf.~\eqref{eqn:matern_on_manifolds}). Namely, as discussed in~\Cref{appdx:heat kernel}, we have
\[ \label{eqn:vector_heat_spectral}
\!\!\v{\c{P}}_t(x, x')=\sum_{n = 0}^{\infty} e^{-t\lambda_n} s_n(x)\otimes s_n(x'),
\]
Here, the notation means
\[
(s_n(x) \otimes s_n(x'))(u, v) = u(s_n(x)) \, v(s_n(x'))
\]
for $u \in T^*_{x}\M$ and $v \in T^*_{x'}\M$.

\subsection{Hodge--Matérn Kernels} \label{sec:hodge-matern}

Following \textcite{azangulov2022, azangulov2023}, we define Matérn kernels $\tk_{\nu, \kappa, \sigma^2}$ as integrals of the heat kernel:
\[ \label{eqn:matern_dfn}
\!\!\!\tk_{\nu, \kappa, \sigma^2}(x, x')
\!=\!
\frac{\sigma^2}{C_{\nu, \kappa}}
\!
\int_0^{\infty}
\!\!\!\!\!
t^{\nu - 1 + \frac{n}{2}}
e^{-\frac{2 \nu}{\kappa^2} t}
\v{\c{P}}_t(x, x')
\d t.
\]
Fubini's theorem then readily implies the key formula
\[ \label{eqn:hodge_matern_kernels}
\tk_{\nu, \kappa, \sigma^2}(x, x')
=
\frac{\sigma^2}{C_{\nu, \kappa}}
\sum_{n=0}^{\infty}
\Phi_{\nu, \kappa}(\lambda_n)
s_n(x) \otimes s_n(x'),
\]
where $\Phi_{\nu, \kappa}$ is as in~\eqref{eqn:phi_dfn} and $C_{\nu, \kappa}$ is a normalizing constant that ensures $\tfrac{1}{\operatorname{vol}\M}\int_\M \tr (\tk_{\nu, \kappa, \sigma^2}(x, x)) \d x = \sigma^2$.
Same as the for Hodge heat kernel, these Hodge--Matérn kernels determine Gaussian vector fields.

\begin{restatable}{theorem}{ThmExistenceGVFMaternKernel}
For any $\nu, \kappa, \sigma^2 > 0$ there exists a Gaussian vector field whose kernel is $\tk_{\nu, \kappa, \sigma^2}$. 
\end{restatable}

In practice, the series in~\Cref{eqn:vector_heat_spectral} should be truncated, with only a few terms corresponding to the smallest eigenvalues $\lambda_n$ used to approximately compute the kernel, just like in the scalar case. Notice that the functions $\Phi_{\nu,\kappa}(\lambda)$ are all decreasing in $\lambda$, so that the most significant terms are the ones corresponding to the smallest eigenvalues.

\subsection{Divergence-Free and Curl-Free Kernels}

The celebrated \emph{Helmholtz decomposition} (also known as the \emph{fundamental theorem of vector calculus}) states that any vector field in $\R^d$ decomposes into the sum of its \emph{divergence-free} and \emph{curl-free} parts.
The former, intuitively, has no sinks and sources; the latter has no vortexes.
Many vector fields in physics are known to have only one of these parts.
This suggests that divergence-free and curl-free Gaussian vector fields can be useful inductive biases \cite{berlinghieri2023}.

For manifolds, the analog of Helmholtz decomposition is the Hodge decomposition, see e.g.~\textcite[Theorem~1.37]{rosenberg1997laplacian}.
It states that any vector field $u$ on $M$ can be represented as a sum of three fields:
\[
u = u_1 + u_2 + u_3,
\]
where $u_1=\nabla f_1$ for some function $f_1$, and thus is \emph{pure divergence}---meaning that $\div u_1\neq0$ and $\curl u_1=0$---and in particular \emph{curl-free}, $u_2=\star\nabla f_2$, and thus is \emph{pure curl} and \emph{divergence-free}, and $u_3$ is a \emph{harmonic form}, $\Delta u_3=0$, both curl- and divergence-free. The symbol $\star$ denotes the \emph{Hodge star} operator, which we recall in \Cref{appdx:hodge star}. For intuition on divergence and curl, see~\Cref{appdx:grad-div-curl on surfaces}.

Importantly, the orthonormal basis of eigenfields $\cbr{s_n}_{n=0}^{\infty}$ may be chosen in such a way that each $s_n$ is in exactly one of the three classes above.
Let $\c{I}_{\div}$, $\c{I}_{\curl}$, and $\c{I}_{\mathrm{harm}}$ denote the index sets of the respective classes of eigenfields.
Using a single class, we can define versions of Matérn Gaussian vector fields on the manifold $M$ with the associated inductive bias.

\begin{restatable}{theorem}{ThmDivCurlGVF}\label{thm:refined kernels}
There exists a Gaussian vector field
$f^{\bdot}$---where $\bdot \in \cbr{\div, \curl, \mathrm{harm}}$
---with kernel
\[
\!\!\tk_{\nu, \kappa, \sigma^2}^{\bdot}(x, x')
=
\frac{\sigma^2}{C_{\nu, \kappa}}
\sum_{n \in \c{I}_{\bdot}}
\Phi_{\nu, \kappa}(\lambda_n)
s_n(x) \otimes s_n(x').
\]
What is more,
$\div f^{\curl} = 0$, $\curl f^{\div} = 0$ and $\Delta f^{\mathrm{harm}} = 0$
almost surely as long as $f^{\bdot}$ is smooth enough that $\div$ and $\curl$ are well-defined.
\end{restatable}

\subsection{Hodge-compositional Matérn Kernels} \label{sec:hodge-comp}

Combining the pure divergence, pure curl, and harmonic kernels, each with a separate set of hyperparameters, gives a more flexible family of kernels:
\[
\sigma_1^2 \tk_{\nu, \kappa_1, 1}^{\div} + \sigma_2^2 \tk_{\nu, \kappa_2, 1}^{\curl} + \sigma_3^2 \tk_{\nu, \kappa_3, 1}^{\mathrm{harm}}.
\]
By analogy with the concurrent paper by~\textcite{yang2024}, we call this family \emph{Hodge-compositional Matérn kernels}.
Unless prior knowledge suggests a more specialized choice, this is the family we recommend for use in practical applications, inferring all the hyperparameters $\kappa_i,\sigma_i$ from data.
Our experimental results in~\Cref{sec:experiments} support this recommendation.

\subsection{Gaussian Process Regression}

All kernels defined above fall under the umbrella framework of Gaussian vector fields described in \textcite{hutchinson2021}.
In practice, there are two ways to perform Gaussian process regression in this setting.
The first one---if the manifold is embedded in $\R^D$---is to treat Gaussian vector fields as special cases of Gaussian vector functions in $\R^D$.
The second is to introduce a \emph{frame}, i.e. a coordinate choice (not necessarily smooth) in all of the tangent spaces, and describe all quantities in these coordinates.
Depending on whether an embedding or a frame is available, either can be used. 
Both are merely modes of computation, not affecting the inductive biases of Gaussian vector fields and not introducing any error per se.

\subsection{Kernel Evaluation and Sampling}

\begin{figure*}[t!]%
\begin{subfigure}{0.25\textwidth}%
\centering%
\includegraphics[width=\textwidth]{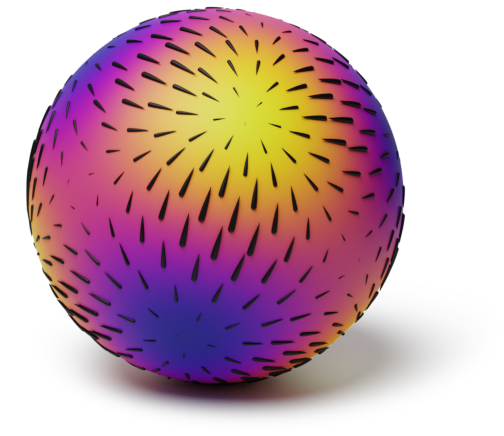}%
\caption{Eigenfield $\frac{1}{\sqrt{\lambda_3}}\nabla Y_{3,2}$}%
\end{subfigure}%
\begin{subfigure}{0.25\textwidth}%
\centering%
\includegraphics[width=\textwidth]{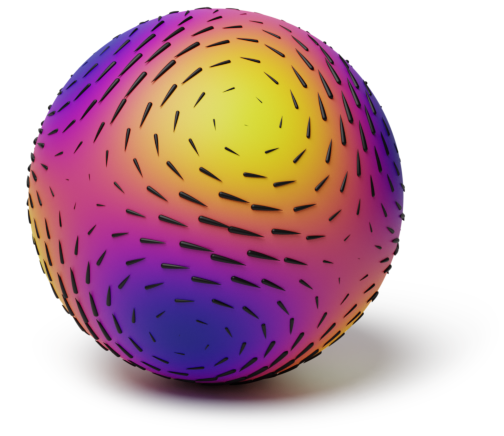}%
\caption{Eigenfield $\frac{1}{\sqrt{\lambda_3}}\star\nabla Y_{3,2}$}%
\end{subfigure}%
\begin{subfigure}{0.25\textwidth}%
\centering%
\includegraphics[width=\textwidth]{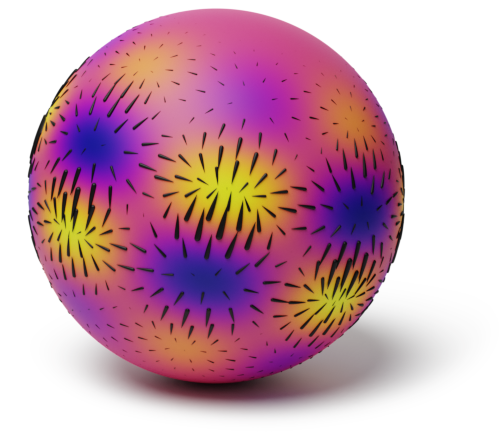}%
\caption{Eigenfield $\frac{1}{\sqrt{\lambda_7}}\nabla Y_{7,3}$}%
\end{subfigure}%
\begin{subfigure}{0.25\textwidth}%
\centering%
\includegraphics[width=\textwidth]{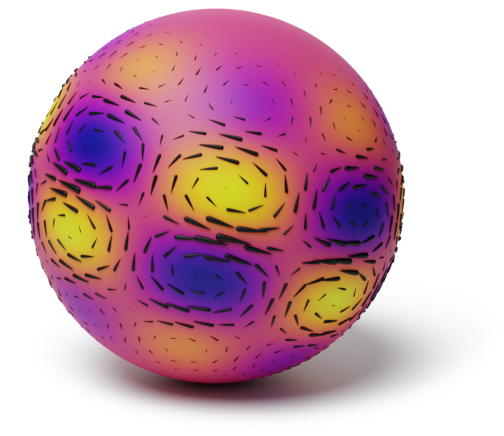}%
\caption{Eigenfield $\frac{1}{\sqrt{\lambda_7}}\star\nabla Y_{7,3}$}%
\end{subfigure}%
\caption{Eigenfunctions on the sphere $\mathbb{S}_2$ (represented by color) and the respective eigenfields.}%
\label{fig:eigenfunctions_eigenfields}%
\end{figure*}

As mentioned in the end of~\Cref{sec:hodge-matern}, given that eigenfields and eigenvalues are known, we can approximately evaluate the kernels by truncating the series in~\Cref{eqn:hodge_matern_kernels}.
Such a truncation is a well-defined kernel, i.e.~it corresponds to a Gaussian vector field.

\begin{restatable}{proposition}{ThmSampling} \label{thm:sampling}
The Gaussian vector field
\[ \label{eqn:matern_gvf}
f(x)
=
\frac{\sigma}{\sqrt{C_{\nu, \kappa}}}
\sum_{n=0}^{L}
w_n
\sqrt{\Phi_{\nu, \kappa}(\lambda_n)}
s_n(x),
\]
where $w_n \stackrel{\text{iid}}{\sim} \c{N}(0, 1)$, corresponds to the kernel given by the truncation of~\Cref{eqn:hodge_matern_kernels} with the sum $\sum_{n=0}^{\infty}$ therein substituted by the sum $\sum_{n=0}^{L}$. 
\end{restatable}

Importantly,~\Cref{eqn:matern_gvf} allows to approximately sample Hodge--Matérn Gaussian vector fields in an extremely computationally efficient way by simply drawing random $w_n \sim N(0,1)$.
Efficiently sampling their respective posteriors can be performed using \emph{pathwise conditioning} for Gaussian vector fields, as described in \textcite{hutchinson2021}.

\begin{remark} \label{rem:sampling}
Of course, direct analogs of~\Cref{eqn:matern_gvf} also hold for the kernels of \Cref{thm:refined kernels}.
\end{remark}

In summary, (approximate) sampling, kernel evaluation and differentiation reduce to knowing eigenvalues and eigenfields of the Hodge Laplacian on $M$.
Thus, we proceed to discuss how to obtain those in practice.

\section{\bfseries\small EXPLICIT EIGEN-VALUES AND -FIELDS}

The above allows defining intrinsic kernels on general compact oriented Riemannian manifolds. However, actually computing these kernel requires solving for eigenfields and eigenvalues of the Hodge Laplacian. Luckily, in some important cases this turns out to be tractable. We present them in this section.

\subsection{Surfaces and the Sphere}

The main case of interest here is the sphere $\mathbb{S}_2$.
However, we start by considering the more general case of manifolds of dimension 2 (surfaces).
We explain how to obtain the eigenfields and eigenvalues granted their scalar counterparts and a basis of harmonic vector fields are known.

\paragraph{Surfaces} Suppose $\M$ is a compact, oriented Riemannian surface.
We consider two intrinsic operators for this case.
The first is the gradient of a scalar function, giving us a vector field.
The second operator is the \emph{Hodge star} operator $\star$ acting on vector fields, which in the case of surfaces is just a $90^\circ$ rotation of a vector field in the positive direction, as shown in~\Cref{appdx:grad-div-curl on surfaces}.

Suppose we know all eigenfunctions $\{f_n\}_{n \geq 0}$ and their respective eigenvalues $0=\lambda_0<\lambda_1\le\lambda_2\le\cdots$ of the Laplace--Beltrami operator on $\M$.
Further, assume that we have an orthonormal basis $\{g_j\}_{0\le j\le J}$ of the $0$-eigenspace of the Hodge Laplacian.\footnote{By the Hodge decomposition theorem, these form a basis of the first de Rham cohomology group (a real vector space) of $\M$, which is finite dimensional since $\M$ is compact by assumption.}

\begin{restatable}{theorem}{ThmBasisVFSurface}\label{thm:basis of eigen-vector fields}
For each $n \geq 1$, both $\nabla f_n$ and $\star\nabla f_n$ are eigenfields of the Hodge Laplacian, and the set
\[
\bigg\{\frac{\nabla f_n}{\sqrt{\lambda_n}}, \frac{\star\nabla f_n}{\sqrt{\lambda_n}}, g_j\,\bigg|\,n \geq 1\text{ and }0\le j\le J\bigg\}
\]
forms an orthonormal basis of $\LTM$.
\end{restatable}

All of these operators can easily be computed numerically, e.g.~via automatic differentiation, which makes pointwise evaluation and differentiation of the kernels an easy endeavor with modern computing systems.

\paragraph{The Sphere}
It is well known that the eigenfunctions of the Laplace--Beltrami operator on the sphere $\mathbb{S}_2$ are given by the \emph{spherical harmonics} $Y_{\ell,m}$, for $\ell\ge0,-\ell\le m\le\ell$, with eigenvalues $\lambda_\ell=\lambda_{\ell,m}=\ell(\ell+1)$.
Additionally, the $0$-eigenspace of the Hodge Laplacian on the sphere is trivial, see \Cref{appndx:hodge laplacian}.
This, together with~\Cref{thm:basis of eigen-vector fields}, allows us to compute the eigenfields. We visualize some of them in~\Cref{fig:eigenfunctions_eigenfields}.

The approximation of the full kernel can be made even more efficient via the \emph{addition theorem}, e.g.~\textcite[Section~7.3]{de2021reproducing}, that states
\[
\sum_{-\ell\le m\le\ell}Y_{\ell,m}(x)Y_{\ell,m}(x')=\frac{2\ell + 1}{4\pi}P_\ell(x\cdot x'),
\]
where $P_\ell$ is the $\ell$-th Legendre polynomial and the scalar product is taken in $\R^3$ after embedding the $\mathbb{S}_2$ as the standard unit sphere. This reduces the computations to a single simple function for each eigenvalue.
As a result, we have the following.

\begin{proposition}
    Writing
    \[
    \widetilde{P}_{\ell,\nu,\kappa}(z)=\frac{2\ell + 1}{4\pi\lambda_\ell}\Phi_{\nu, \kappa}(\lambda_\ell)P_\ell(z),
    \]
    the pure divergence and pure curl Hodge--Matérn kernels on $\mathbb{S}_2$ are given by
    \begin{align*}
        \tk^{\div}_{\nu, \kappa,\sigma^2}(x, x')={}&\frac{\sigma^2}{C^{\div}_{\nu, \kappa}}\sum_{\ell = 1}^{\infty}(\nabla_x\otimes\nabla_{x'})\widetilde{P}_{\ell,\nu,\kappa}(x\cdot x'),\\
        \tk^{\curl}_{\nu, \kappa,\sigma^2}(x, x')={}&\frac{\sigma^2}{C^{\curl}_{\nu, \kappa}}\sum_{\ell = 1}^{\infty}(\star\nabla_x\otimes\star\nabla_{x'})\widetilde{P}_{\ell,\nu,\kappa}(x\cdot x').
    \end{align*}
    The full Hodge--Matérn kernel is the mean of the two:
    \[
    \tk_{\nu, \kappa, \sigma^2} = \frac{1}{2}\left(\tk^{\div}_{\nu, \kappa,\sigma^2} + \tk^{\curl}_{\nu, \kappa,\sigma^2}\right).
    \]
\end{proposition}

Combining this with~\Cref{thm:sampling}, we can sample Hodge--Matérn Gaussian vector fields.
Thanks to~\Cref{rem:sampling}, we can also sample their pure divergence and pure curl counterparts, which we illustrate in~\Cref{fig:samples_sphere}.

\begin{figure}[t!]%
\begin{subfigure}{0.5\linewidth}%
\includegraphics[width=\linewidth]{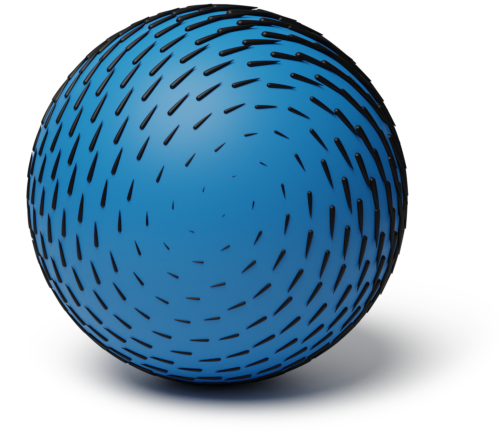}%
\caption{Divergence-free field $f^{\curl}$}%
\end{subfigure}%
\begin{subfigure}{0.5\linewidth}%
\includegraphics[width=\linewidth]{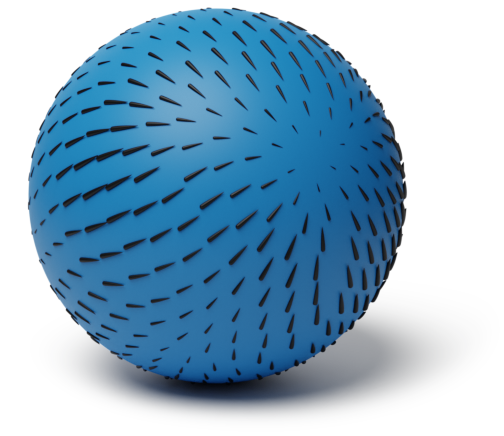}%
\caption{Curl-free field $f^{\div}$}%
\end{subfigure}%
\caption{Intrinsic Gaussian vector field samples on~$\mathbb{S}_2$.}%
\label{fig:samples_sphere}%
\end{figure}

\subsection{Product Manifolds and Hypertori}

Another tractable setting is product manifolds: if we know the eigenfields of the Hodge Laplacian for the factors, we can construct those of the product. We give an overview of this in general before deriving the spectrum of the circle from the scalar case and using it to resolve the case of the hypertori.

\paragraph{Product Manifolds} Let $\M,\NN$ be two compact, oriented Riemannian manifolds with scalar manifold heat kernels $\c{P}^{\M}, \c{P}^{\NN}$ and Hodge heat kernels $\v{\c{P}}^{\M}, \v{\c{P}}^{\NN}$. The vector kernel on the product $\M\times\NN$ is given by
\[
\begin{aligned}
    \v{\c{P}}_t^{\M\times\NN}(x, x') ={}& \c{P}_t^{\M}(x_1,x_1')\v{\c{P}}_t^{\NN}(x_2,x_2')\,+\\
    &+ \c{P}_t^{\NN}(x_2,x_2')\v{\c{P}}_t^{\M}(x_1,x_1')
\end{aligned}
\]
for $x=(x_1,x_2)$ and $x' = (x_1', x_2')$ and $x_1, x_1' \in \M$; $x_2, x_2' \in \NN$. The details are explained in \Cref{appdx:prod mflds}.

Knowing the Hodge heat kernels, the other Hodge--Matérn kernels can be derived using~\Cref{eqn:matern_dfn}.

\paragraph{Circle} The circle $\mathbb{S}_1$ is the only\footnote{To be precise: any $1$-dimensional compact Riemannian manifold is isometric to the circle of the same length.} compact Riemannian manifold of dimension $1$. The Hodge star operator in this case gives an identification of scalar functions and vector fields on $\mathbb{S}_1$: there is a canonical global vector field $v$ such that $\norm{v(x)} \equiv 1$---in fact, there are exactly two of them, and a choice of orientation selects one---and a function $f(x)$ on $\mathbb{S}_1$ is identified with $f(x)v(x)$.
The Hodge Laplacian reduces to the Laplace--Beltrami operator
\[
\Delta(f(x)v(x)) = \Delta(f(x))v(x).
\]
Thus, the spectrum of the Hodge Laplacian is the same as the spectrum of the Laplace--Beltrami operator under this identification. Similarly, the vector kernels coincide with their scalar counterparts, for which explicit formulas can be found in \textcite{borovitskiy2020}.

\begin{figure*}[t!]
\begin{subfigure}{0.333\textwidth}%
\centering%
\includegraphics[width=0.75\textwidth]{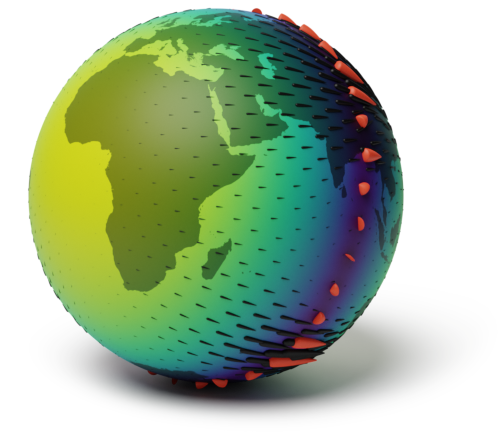}%
\caption{Projected Mat\'ern ($\nu=\tfrac{1}{2}$)}%
\label{fig:winds proj mean}%
\end{subfigure}%
\hfill%
\begin{subfigure}{0.333\textwidth}%
\centering%
\includegraphics[width=0.75\textwidth]{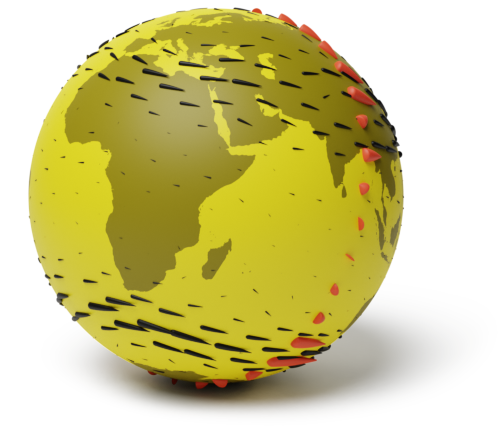}%
\caption{Ground truth (January 2010)}%
\label{fig:winds ground truth}%
\end{subfigure}%
\hfill%
\begin{subfigure}{0.333\textwidth}%
\centering%
\includegraphics[width=0.75\textwidth]{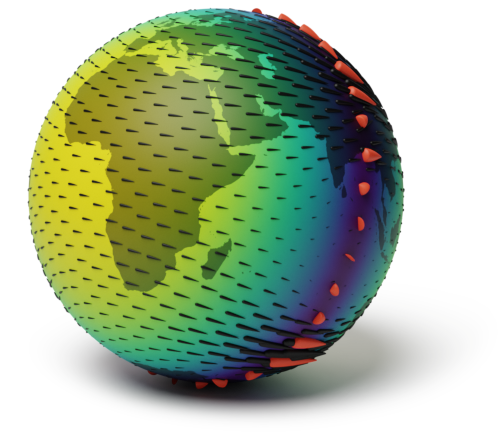}%
\caption{Div.-free Hodge-Mat\'ern ($\nu=\tfrac{1}{2}$)}%
\label{fig:winds hodge mean}%
\end{subfigure}%
\caption{Interpolation of wind speed on the surface of Earth. The observations are the red vectors along a meridian. Figures (a) and (c) report predictive mean (black vectors) and uncertainty (color: yellow is high, blue is low). Note that in figure (a) sinks and sources are present, while the inductive bias of (c) prohibits that. We advise the reader to examine the global \Cref{fig:flat wind modeling proj,fig:flat wind modeling hodge}, located in the appendix because of space~limitations.}
\label{fig:weather_modeling}
\end{figure*}

\paragraph{Hypertori} The  $d$-dimensional (flat\footnote{Not to be confused with the torus defined as a "donut" in $\R^3$, which has a different intrinsic geometry.}) hypertorus is defined as the product of $d$ circles
$\T^d\coloneqq\mathbb{S}_1\times\cdots\times \mathbb{S}_1$. It has a global basis of tangent vector fields by the canonical identification $T\T^d\cong T\mathbb{S}_1\otimes\cdots\otimes T\mathbb{S}_1$, so that we can write tangent fields on $\T^d$ as a vector of vector fields on $\mathbb{S}_1$. Putting together what was said above, we obtain the following, cf.~\Cref{appdx:prod mflds}.

\begin{restatable}{proposition}{ThmMaternTorus}
    The eigenfields on $\T^d$ are
    \[
    \left(\prod_{i=1}^df_{n_i}(x_i)\right)\v{e_j}\in T\mathbb{S}_1\otimes\cdots\otimes T\mathbb{S}_1\cong T\T^d,
    \]
    $1\le j\le d$, with eigenvalue $\sum_{i=1}^d\lambda_{n_i}$, where $x_i \in \mathbb{S}_1$, and $(f_n, \lambda_n)$ are eigenpairs of the Laplace--Beltrami operator on $\mathbb{S}_1$. In particular, with $k^{\T^d}$ the scalar Mat\'ern kernel on $\T^d$, the vector Hodge--Mat\'ern kernel is
    \[
    \tk_{\nu,\kappa,\sigma^2}^{\T^d}(x,x')=\frac{1}{d}k_{\nu,\kappa,\sigma^2}^{\T^d}(x_i,x_i')\v{I}_d.
    \]
\end{restatable}

\subsection{Possible Extensions}

We propose two prospective directions into which our results could be extended.

\paragraph{Meshes}
Neither Hodge--Matérn kernels nor the eigenfields can be \emph{analytically} computed on a general two-dimensional manifold. This is true even for their scalar counterparts \cite{borovitskiy2020}.
However, we expect that Hodge--Matérn kernels can be numerically approximated on surfaces discretized into meshes.
To do this, one needs to apply suitable discrete counterparts of $\star$ and $\nabla$ to numerically approximate the scalar eigenfunctions the Laplace--Beltrami operator and also take care of the harmonic forms.

\paragraph{Lie Groups and Related Manifolds} It is possible to obtain the scalar manifold Matérn kernels on homogeneous spaces via the representation theory of their symmetry groups \cite{azangulov2022}.
We conjecture that the the vector case can be treated similarly---in particular, in view of the work of \textcite{ikeda1978spectra}.
This could result in explicit formulas for eigenfields and eigenvalues for homogeneous spaces given in terms of algebraic quantities only.

\section{\bfseries\small EXPERIMENTS}
\label{sec:experiments}

We complement the theoretical motivations for the use of intrinsic kernels on manifolds with a practical experiment on weather data from the ERA5 dataset \cite{era5}. Further experiments on synthetically generated data are available in \Cref{appendix:experimental details}.

\subsection{Setup}

The dataset comes from the fifth generation ECMWF atmospheric reanalysis of the global climate (ERA5) \cite{era5}. We took the monthly averaged reanalysis data for wind ($u$- and $v$-components) at the fixed 500hPa pressure level (corresponding to approximately 5.5km altitude) from January to December 2010. As the wind at the 500hPa pressure level at mid and high latitudes is approximately geostrophic and therefore, approximately divergence-free \cite{holton2004}, we expect the family of divergence-free (i.e.~pure curl) Hodge--Mat\'ern kernels to provide good results.

The training and testing points are distributed along a great circle, as displayed in \Cref{fig:winds ground truth}. An experiment was run for each month of data, totaling 12 experiments for each kernel.

We applied GP regression using pure noise, projected Mat\'ern (P.~M.), Hodge--Mat\'ern (H.--M.), and divergence-free Hodge--Mat\'ern (div-free H.--M.) kernels all with $\nu=\tfrac{1}{2},\infty$.
The hyperparameters are the noise variance for the first kernel, and the length scale $\kappa$, the variance $\sigma^2$, and the noise variance $\sigma^2_\varepsilon$ for the others.
They were optimized by maximizing marginal log-likelihood. After visual inspection of the ground truth, we selected the best performing kernel---the divergence-free Hodge--Mat\'ern kernel with $\nu=\tfrac{1}{2}$---and the projected Mat\'ern-$\frac{1}{2}$ kernel, and ran four further experiments by fixing their length scale to high values $\kappa=0.5$ and~$1$. Importantly, we also experimented with the Hodge-compositional Mat\'ern (H.-C. M.) kernels of~\Cref{sec:hodge-comp} to understand if these could detect the correct inductive bias.\footnote{On the sphere $\mathbb{S}_2$, these are linear combinations of pure-divergence and pure-curl Hodge--Mat\'ern kernels, without the harmonic part which vanishes on $\mathbb{S}_2$}

\begin{figure}[b!]%
\begin{subfigure}{0.5\linewidth}%
\includegraphics[width=\linewidth]{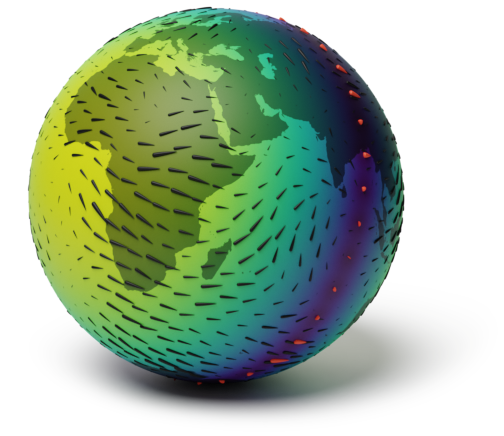}%
\caption{Projected}%
\label{fig:winds proj sample}%
\end{subfigure}%
\begin{subfigure}{0.5\linewidth}%
\includegraphics[width=\linewidth]{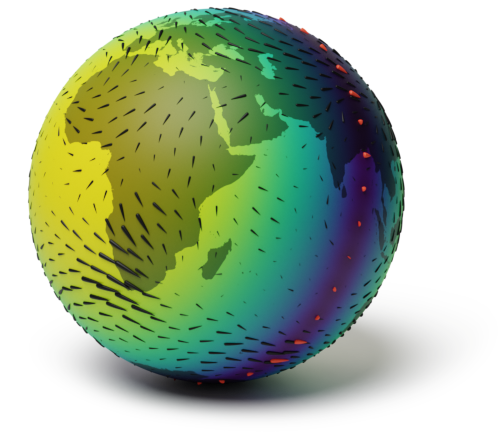}%
\caption{Hodge}%
\label{fig:winds hodge sample}%
\end{subfigure}%
\caption{Posterior samples of the models featured in~\Cref{fig:weather_modeling} (note: different scaling of vectors).}%
\label{fig:weather_samples}%
\end{figure}

We report the mean and standard deviation of mean squared error (MSE) of the predicted mean and predictive negative log-likelihood (PNLL) of the ground truth against the predictive distribution, where each test point is considered independently of the others.

\subsection{Performance}

\begin{table}[t]%
\centering
\begin{tabular}{lrrrr}
\toprule
\multirow{2}{*}{\textbf{Kernel}} & \multicolumn{2}{c}{\textbf{MSE}} & \multicolumn{2}{c}{\textbf{PNLL}} \\
 \cmidrule(lr){2-3} \cmidrule(lr){4-5}& Mean & Std & Mean & Std \\
\midrule
Pure noise & 2.07 & 0.18 & 2.89 & 0.10 \\
P.~M.--$\tfrac{1}{2}$ & 1.39 & 0.15 & 2.33 & 0.11 \\
P.~M.--$\infty$ & 1.53 & 0.20 & 2.43 & 0.15 \\
H.--M.--$\tfrac{1}{2}$ & 1.67 & 0.16 & 2.58 & 0.12 \\
H.--M.--$\infty$ & 1.76 & 0.17 & 2.63 & 0.14 \\
div-free H.--M.--$\tfrac{1}{2}$ & \bf 1.10 & 0.12 & \bf 2.16 & 0.13 \\
div-free H.--M.--$\infty$ & 1.34 & 0.20 & 2.33 & 0.18 \\
\midrule
H.-C.~M.--$\tfrac{1}{2}$ & \bf 1.09 & 0.10 & \bf 2.16 & 0.12 \\
H.-C.~M.--$\infty$ & 1.34 & 0.20 & 2.33 & 0.18 \\
\bottomrule
\end{tabular}

\caption{Mean squared error and predictive negative log-likelihood in the ERA5 wind data experiment.}%
\label{table:results ERA5 orbit}%
\end{table}

Results are displayed in \Cref{table:results ERA5 orbit,table:results ERA5 orbit fixed kappa}. The data was scaled so that the training observations have unit mean norm, cf.~\Cref{appendix:experimental details}. Baselines are provided by fitting a pure noise kernel to the data and the projected kernels of \textcite{hutchinson2021}. The best scores by far were obtained by fitting either of the divergence-free Hodge--Mat\'ern or Hodge-compositional Mat\'ern kernels with $\nu=\tfrac{1}{2}$, which outperformed the other kernels we considered on all metrics. Next to them were divergence-free Hodge--Mat\'ern with $\nu=\infty$ and the projected Mat\'ern kernel with $\nu=\tfrac{1}{2}$.

Two factors were of fundamental importance in obtaining good results: having a divergence-free inductive bias and allowing for lower degrees of smoothness by setting $\nu=\tfrac{1}{2}$. In particular, considering the divergence of the various kernels explains why the projected kernels tend to do better than the full Hodge-Mat\'ern kernels in this situation: they have lower absolute divergence in expectation, cf.~\Cref{appdx:divergence of GPs}, \Cref{fig:var div quantification}.

The additional experiments where length scales were fixed to high values
are reported in \Cref{table:results ERA5 orbit fixed kappa}. We notice further test score improvements, but not significant ones. \Cref{fig:weather_modeling,fig:weather_samples} display posterior mean, standard deviations and samples from GP regressions fitted using projected Mat\'ern and divergence-free Hodge--Mat\'ern kernels, both with fixed $\kappa=1$.

\begin{table}[t!]%
\centering
\begin{tabular}{lrrrr}
\toprule
\multirow{2}{*}{\textbf{Kernel}} & \multicolumn{2}{c}{\textbf{MSE}} & \multicolumn{2}{c}{\textbf{PNLL}} \\
 \cmidrule(lr){2-3}\cmidrule(lr){4-5}& Mean & Std & Mean & Std \\
\midrule
div-free H.--M.--$\tfrac{1}{2}$ & 1.10 & 0.12 & 2.16 & 0.13 \\
div-free H.--M.--$\infty$ & 1.34 & 0.20 & 2.33 & 0.18 \\
H.--M.--$\tfrac{1}{2}$ ~$\kappa=0.5$ & 1.06 & 0.10 & \bf 2.15 & 0.13 \\
H.--M.--$\tfrac{1}{2}$ ~$\kappa=1.0$ & \bf 1.05 & 0.10 & 2.17 & 0.13 \\
P.~M.--$\tfrac{1}{2}$ ~~$\kappa=0.5$ & 1.36 & 0.13 & 2.32 & 0.10 \\
P.~M.--$\tfrac{1}{2}$ ~~$\kappa=1.0$ & 1.36 & 0.13 & 2.34 & 0.10 \\
\bottomrule
\end{tabular}

\caption{Mean squared error and predictive negative log-likelihood in the ERA5 wind data experiment for divergence-free Hodge--Mat\'ern and projected Mat\'ern kernels with $\nu=\tfrac{1}{2}$ and fixed length scale $\kappa$.}%
\label{table:results ERA5 orbit fixed kappa}%
\end{table}

Analysis of the fitted hyperparameters of Hodge-compositional Mat\'ern kernels shows that these were able to automatically detect the correct inductive bias.
Specifically, the resulting Hodge-compositional kernel was virtually divergence-free, which enabled it to reach the same scores as the actual divergence-free Hodge--Mat\'ern kernels.
See \Cref{appendix:experimental details} for more details.

\section{\bfseries\small CONCLUSION}

\begin{samepage}
In this work, we introduced a novel class of models---the \emph{Hodge--Mat\'ern} vector fields---for learning tangential vector fields on manifolds.
These principled models solve the shortcomings present in the preexisting literature, providing improved inductive biases.
We described computational techniques---kernel evaluation and differentiation, sampling---required for running Gaussian process regression and for downstream applications on important manifolds such as the two-dimensional sphere $\mathbb{S}_2$ and hypertori $\mathbb{T}^d$, and indicated further extension directions.
We applied our methods both to synthetic and real-world data on $\mathbb{S}_2$, demonstrating that they can lead to improved performance.
We hope our results will inspire new extensions and applications in areas like climate modeling and robotics.
\end{samepage}

\section*{\bfseries\small ACKNOWLEDGMENTS}

We are grateful to Gergana Gyuleva for her help with retrieving meaningful weather data for our experiments.
We thank Maosheng Yang and Dr.~So Takao for their insightful input on the drafts of the paper.
We also thank Dr.~Alexander Terenin for his guidance in using his Blender rendering scripts of his Ph.D.~thesis repository.\footnote{\url{https://github.com/aterenin/phdthesis} \cite{terenin2022}.$\!\!\!\!$}
Finally, we thank Dr.~Alan Pinoy for providing the proof of \Cref{lemma:div proj}. VB acknowledges support by an ETH Zürich Postdoctoral Fellowship.

\newpage

\printbibliography[title=\bfseries\small REFERENCES]

\clearpage
\appendix

\allowdisplaybreaks

\section{\bfseries\small THEORY} \label{appdx:theory}

In this section, we present the mathematical theory behind the Hodge--Mat\'ern Gaussian vector fields introduced in the main body of the paper. Although we previously almost exclusively talked about vector fields on manifolds, this section is written in the more technically appropriate language of differential forms. On a Riemannian manifold, the metric provides an equivalence between $1$-forms and vector fields, which immediately recovers the results of the main part of the article. The general theory provides more than just intrinsic kernels for $1$-forms, though: we obtain kernels for differential forms of all degrees---hence intrinsic Gaussian differential forms---which can be refined to kernels for pure gradient, pure curl, and harmonic processes.

We will begin by introducing the Hodge star operator and link the abstract differential operators to the classical notions of gradients, curls, and divergences. The Hodge star operator will allow us to define the Hodge Laplacian and the heat equation on differential forms. The fundamental solution of the latter is the heat kernel (up to a normalization constant). We will then explore the cases of surfaces and products, and conclude by giving an overview of another possible generalization of the Laplacian to sections of a vector bundle that could be considered in future research.

We assume the reader has some basic familiarity with concepts of differential geometry, such as differential forms and the exterior derivative. All manifolds are assumed to be compact, connected, and oriented.

Most of the material found in this appendix can be found in the book by \textcite{rosenberg1997laplacian}. We only give proofs of original results.

\subsection{The Hodge Star Operator}\label{appdx:hodge star}

Let $V$ be a finite-dimensional oriented vector space, say of dimension $n$, endowed with a non-degenerate inner product $\langle\,\cdot\,,\,\cdot\,\rangle$.

The $k$-th \emph{exterior product} $\Lambda^kV$ of $V$ is the quotient of the tensor product of $k$ copies of $V$ by the sign representation of the symmetric group:
\[
v_{\sigma(1)}\otimes\cdots\otimes v_{\sigma(k)}\sim\operatorname{sgn}(\sigma)v_1\otimes\cdots\otimes v_k
\]
for $v_1,\ldots,v_k\in V$ and $\sigma\in\Sigma_k$ a permutation. We write $v_1\wedge\cdots\wedge v_k$ for the equivalence class of $v_1\otimes\cdots\otimes v_k$, and we call $\wedge$ the \emph{wedge product}.

The inner product on $V$ induces an  inner product on the exterior product $\Lambda^kV$ for $0\le k\le n$ by
\[
\langle a,b\rangle=\det\left(\langle a_i,b_j\rangle_{i,j}\right)
\]
for $a=a_1\wedge\ldots\wedge a_k$ and $b=b_1\wedge\ldots\wedge b_k$.

It is easily checked that $\Lambda^nV$ is $1$-dimensional and has a canonical unit $n$-vector
\[
\omega = e_1\wedge\ldots\wedge e_n\in\Lambda^nV,
\]
where $e_1,\ldots,e_n\in V$ is an oriented orthonormal basis.

The Hodge star operator is the linear operator
\[
\star:\Lambda^kV\longrightarrow\Lambda^{n-k}V
\]
mapping $b\in\Lambda^kV$ to the unique $\star b\in\Lambda^{n-k}V$ such that
\[
a\wedge(\star b) = \langle a,b\rangle\omega
\]
for any $a\in\Lambda^kV$.

From now on, let $\M$ be a compact oriented Riemannian manifold of dimension $d_{\M}$. The cotangent bundle $T^*\M$ is the dual of the tangent bundle $T\M$. In other words, at each $x\in\M$ the fibre $T^*_x\M$ is the linear dual of the tangent space $T_x\M$ at $x$. One can then take the exterior products of the cotangent bundle, which are nothing more than the exterior product $\Lambda^k T^*_x\M$ at each $x\in\M$. Sections of this bundle---smooth functions that at each point $x\in\M$ associate an element of $\Lambda^k T^*_x\M$---are called \emph{$k$-differential forms} on $\M$. The space of $k$-differential forms is denoted by $\Omega^k(\M)$.

The exterior derivative is an intrinsic differential operator
\[
\d:\Omega^k(\M)\longrightarrow\Omega^{k+1}(\M)
\]
defined in local coordinates by linear extension of
\[
\d(f(x)\d x_I)=\sum_{i=1}^{d_{\M}}\frac{\partial f(x)}{\partial x_i}\d x_i\wedge\d x_I,
\]
where $I=\{i_1,\ldots,i_k\}\subseteq\{1,\ldots,d_{\M}\}$ is a set of $k$ indices and $\d x_I=\d x_{i_1}\wedge\cdots\wedge\d x_{i_k}$. In order to build intuition, one can notice that if $X$ is a vector field on $\M$ and $f\in\Omega^0(\M)$, then  $\d f(X)$ is the directional derivative of $f$ in direction $X$.

Applying the Hodge star operator to the cotangent bundle at each point of $\M$ (the inner product being given by the Riemannian metric), one obtains a Hodge star operator on differential forms
\[
\star:\Omega^k(M)\longrightarrow\Omega^{n-k}(M).
\]
The canonical unit $n$-vector in $\Lambda^nT^*\M$ is nothing else than the Riemannian volume form $\omega_g$, and by definition we have $\star1=\omega_g$ and $\star\omega_g=1$.

If $\alpha,\beta\in\Omega^k(\M)$ are two $k$-forms, their (Hodge) inner product is
\[
\langle\alpha,\beta\rangle_{\c{L}^2(\M)}=\int_\M\langle\alpha,\beta\rangle\omega_g=\int_\M\alpha\wedge\star\beta.
\]
An easy computation shows that the dual of the exterior derivative $\d:\Omega^k(\M)\to\Omega^{k+1}(\M)$ is given by
\[
\d^\star=(-1)^{d_{\M}k+1}\star \d\star:\Omega^{k+1}(\M)\longrightarrow\Omega^k(\M).
\]

\subsection{Gradient, Divergence, and Curl on Surfaces}\label{appdx:grad-div-curl on surfaces}

In this section, we will link the abstract operators $\star$, $\d$, and $\d^\star$ to quantities which are in some sense more concrete. Before doing that, we start with a quick recap on vector fields, their curls and divergences in~$\R^3$.

In $\R^3$, a basis for $1$-forms is given by $\d x_1,\d x_2,\d x_3$, and a basis for $2$-forms is $\d x_1\wedge\d x_2,\d x_1\wedge\d x_3,\d x_2\wedge\d x_3$. The Hodge star operator sends the constant function $1$ to the volume form $\d x_1\wedge\d x_2\wedge\d x_3$, and it maps
\[
    \star\d x_1={}&\d x_2\wedge\d x_3,\\
    \star\d x_2={}&-\d x_1\wedge\d x_3,\\
    \star\d x_3={}&\d x_1\wedge\d x_2,
\]
and vice versa. Vector fields are identified with $1$-forms by mapping $\d x_i\mapsto\v{e}_i$, and also with $2$-forms via the Hodge star operator.

With this in mind, we immediately see that the exterior derivative of a function
\[
\d f = \partial_1f\d x_1 + \partial_2f\d x_2 + \partial_3f\d x_3
\]
is identified with taking the gradient $\nabla f$. A straightforward calculation also shows that taking the exterior derivative of a $1$-form corresponds to taking the curl of the corresponding vector field, and the exterior derivative of a $2$-form gives the divergence of the associated vector field. This and the fact that $\d^2=0$ recover the well known relations between $\nabla, \curl$, and $\div$. Via the Hodge star operator, similar statements can be made about the $\d^\star$ operator.

On surfaces, the situation is a bit different since we only have two dimensions in which to move. The fact that our manifold is oriented gives us an orientation on each cotangent space $T^*_x\M$. Picking an oriented orthonormal local basis $\d x_1,\d x_2\in T^*_x\M$, the volume form is given locally by $\d x_1\wedge\d x_2$ and it follows that $\star\d x_1=\d x_2$ and $\star\d x_2=-\d x_1$. This corresponds to a rotation by $90^\circ$ in the cotangent space.

In the case where our surface is embedded in $\R^3$, the choice of an orientation is equivalent to the choice of a global unit normal field for the manifold, i.e.~a smooth choice of a unit normal vector for each point. Then, for tangent vectors the Hodge star is given extrinsically by a rotation by $90^\circ$ around the unit normal at each point. This can also be written as the cross product of a tangent vector with the unit normal.

Building on the case of $\R^3$, we call
\begin{align}
    \grad = \d{}&: \Omega^0(\M)\longrightarrow\Omega^1(\M),\\
    \div = \d^\star{}&: \Omega^1(\M)\longrightarrow\Omega^0(\M),\\
    \curl = \star\d{}&: \Omega^1(\M)\longrightarrow\Omega^0(\M),
\end{align}
under the identification of $1$-forms with vector fields and $2$-forms with functions. This also corresponds---potentially up to a sign, depending on conventions---to the definition of divergence in Riemannian geometry using the Levi--Civita connection \cite[Equation 13.11]{lee2009manifolds}. Working in a local coordinate system on $\M$, the explicit expression for these operators is
\begin{align}
    \nabla f(x) ={}&\begin{pmatrix}
        \partial_1f(x)\\
        \partial_2f(x)
    \end{pmatrix},\\
    \div\begin{pmatrix}
        v_1(x)\\v_2(x)
    \end{pmatrix}={}&\partial_1v_1(x) + \partial_2v_2(x),\\
    \curl\begin{pmatrix}
        v_1(x)\\v_2(x)
    \end{pmatrix}={}&\partial_1v_2(x) - \partial_2v_1(x).
\end{align}

\subsection{The Hodge Laplacian}\label{appndx:hodge laplacian}

Let $\M$ be a $d$-dimensional manifold. 
The \emph{Hodge Laplacian} on differential forms is then defined as
\[
\Delta\coloneqq -\left(\d^\star \d + \d\d^\star\right).
\]

\begin{remark}
    In some texts, including \textcite{rosenberg1997laplacian}, the Hodge Laplacian has the opposite sign.
\end{remark}

\begin{remark}
    For $k=0$ (i.e.~on functions) this recovers the Laplace--Beltrami operator
    \[
    \Delta=-\d\d^\star=-\nabla^*\nabla,
    \]
    where $\nabla$ is the Levi--Civita connection. Cf.~also \cref{subsection:Bochner Laplacian}. In the dual context of vector fields, we obtain the classical divergence of the gradient
    \[
    \Delta=\div \nabla.
    \]
\end{remark}

The following deep results---found in \textcite[Theorems~1.30, 1.37, and 1.45]{rosenberg1997laplacian}---give us all we need to know about the spectrum of the Hodge Laplacian.

\begin{theorem}[Hodge]
    All the eigenvalues of the Hodge Laplacian $\Delta$ on $\Omega^k(\M)$ are non-negative, they have finite multiplicity, and they accumulate only at infinity. The eigenforms span a dense subset of $\Omega_{\c{L}^2}(\M)$. In particular, there exists an orthonormal basis of $\Omega_{\c{L}^2}(\M)$ consisting of smooth eigenforms of $\Delta$.
\end{theorem}

\begin{remark}
    The convention on eigenpairs is that $\phi$ is an eigenform of (minus) eigenvalue $\lambda$ if $\Delta\phi=-\lambda\phi$.
\end{remark}

\begin{theorem}[Hodge decomposition]\label{hodge decomposition}
    The space of smooth $k$-forms decomposes as
    \[
    \Omega^k(\M)=\ker\Delta\oplus\im\d\oplus\im\d^\star.
    \]
\end{theorem}

The next result links the kernel of the Hodge Laplacian with a purely topological property of the manifold: de Rham cohomology. An accessible introduction is given in \textcite[Section~1.4]{rosenberg1997laplacian}.

\begin{theorem}[Hodge]
    The kernel of the Hodge Laplacian on $k$-forms is naturally isomorphic to the $k$-th de Rham cohomology group, which is a real vector space:
    \[
    \ker\Delta\cong H^k_{dR}(\M).
    \]
\end{theorem}

\begin{remark}\label{rem:cohomology of sphere}
    The following facts about de Rham cohomology are often useful. Assume $\M$ is compact and connected.
    \begin{enumerate}
        \item $H^0_{dR}(\M)\cong\R$, spanned by the constant function $f(x)=1$.
        \item $H^{d_{\M}}_{dR}(\M)\cong\R$, spanned by the volume form.
        \item All of the $H^k_{dR}(\M)$ are finite dimensional, see e.g.~\textcite[Theorem~10.17]{lee2009manifolds}.
    \end{enumerate}
    An example that is exploited in the main body of this paper is the well known fact that
    \[
    H^1_{dR}(\mathbb{S}_2) = 0.
    \]
    This can be computed using the Mayer--Vietoris sequence, see e.g.~\textcite[Section~10.1]{lee2009manifolds}.
\end{remark}

\begin{proposition}\label{proposition: transformations of eigenforms}
    The various spaces of eigenforms have the following relations.
    \begin{enumerate}
        \item\label{eigenforms relations star} The Hodge star $\star:\Omega^k(\M)\to\Omega^{n-k}(\M)$ sends eigenforms of the Hodge Laplacian to eigenforms with the same eigenvalue, and it preserves their orthogonality and norm.
        \item\label{eigenforms relations d} The exterior derivative $\d:\Omega^k(\M)\to\Omega^{k+1}(\M)$ sends eigenforms in $\im \d^\star$ to eigenforms with the same eigenvalue (and is zero on the other eigenforms), it preserves orthogonality, and
        \[
        \|\d\phi\|_{\c{L}^2(\M)} = \sqrt\lambda\|\phi\|_{\c{L}^2(\M)}
        \]
        for $\phi\in\im \d^\star\subseteq\Omega^k(\M)$ an eigenform of eigenvalue $-\lambda$.
        \item\label{eigenforms relations d star} The operator $\d^\star:\Omega^k(\M)\to\Omega^{k-1}(\M)$ sends eigenforms in $\im \d$ to eigenforms with the same eigenvalue (and is zero on the other eigenforms), it preserves orthogonality, and
        \[
        \|\d^\star\phi\|_{\c{L}^2(\M)} = \sqrt\lambda\|\phi\|_{\c{L}^2(\M)}
        \]
        for $\phi\in\im \d\subseteq\Omega^k(\M)$ an eigenform of eigenvalue~$-\lambda$.
    \end{enumerate}
\end{proposition}

\begin{proof}
    For \eqref{eigenforms relations star}, it is immediate to see that $\Delta\star = \star\Delta$ so that $\star$ sends eigenforms to eigenforms with the same eigenvalue. If $\alpha,\beta$ are any two $k$-forms, then
    \[
        \langle\star\alpha,\star\beta\rangle_{\c{L}^2(\M)}={}&\int_\M\star\alpha\wedge\star\star\beta\\
        ={}&\int_\M\star\alpha\wedge\beta\\
        ={}&\int_\M\beta\wedge\star\alpha = \langle\beta,\alpha\rangle_{\c{L}^2(\M)}\\
        ={}&\langle\alpha,\beta\rangle_{\c{L}^2(\M)},
    \]
    showing that $\star$ is an isometry and concluding the point.

    For \eqref{eigenforms relations d}, one easily checks that $\Delta \d=\d\Delta$ so that $\d$ sends eigenforms that are in $\im \d^\star$, i.e.~the complement of $\ker \d$, to eigenforms with the same eigenvalue. Let $\phi,\psi\in\im \d^\star$ be eigenforms with eigenvalues $\lambda_\phi,\lambda_\psi$ respectively, then
    \[
    \left\langle \d\phi,\d\psi\right\rangle_{\c{L}^2(\M)}={}&\left\langle \phi,\d^\star \d\psi\right\rangle_{\c{L}^2(\M)}\\
    ={}&\left\langle \phi,-\Delta\psi\right\rangle_{\c{L}^2(\M)}\\
    ={}&\lambda_\psi\left\langle \phi,\psi\right\rangle_{\c{L}^2(\M)},
    \]
    where in the middle equality we used the fact that $\psi\in\im \d^\star$ and that $\d^\star \d^\star=0$. This concludes the proof of the point, and \eqref{eigenforms relations d star} is analogous.
\end{proof}

\subsection{The Heat Equation and its Kernel}\label{appdx:heat kernel}

The heat equation for $k$-forms can now be defined as
\[
\partial_t\alpha(t,x) = \Delta_x\alpha(t,x)
\]
with a given initial condition $\alpha(0,x) = \beta(x) \in\Omega^k(\M)$.

A \emph{double ($k$-)form} over $\M$ is a smooth section of the bundle $\R\otimes\Lambda^kT^*\M\otimes\Lambda^kT^*\M$ over $\M\times\M$, where the fibre above $(x,x')\in\M\times\M$ is $\R\otimes\Lambda^kT^*_x\M\otimes\Lambda^kT^*_{x'}\M$.

A \emph{heat kernel} for $k$-forms is a double form $\c{P}_t(x,y)$ such that
\begin{enumerate}
    \item $(\partial_t-\Delta_x)\c{P}_t(x,x')=0$ and
    \item $\lim_{t\to\infty}\int_M\left\langle \c{P}_t(x,x'),\alpha(x')\right\rangle_{x'}\d x' = \alpha(x)$ for any $\alpha\in\Omega^k(\M)$, where the pointwise inner product and the integration are taken with respect to $x'$, and integration is against the volume form of $\M$.
\end{enumerate}

\begin{theorem}\label{thm: heat kernel as a sum}
    Let $\phi_i\in\Omega^k(\M)$ be an orthonormal basis of $k$-eigenforms of the Hodge Laplacian with (minus) eigenvalues $0\le\lambda_1\le\lambda_2\le\cdots$. The heat kernel on $k$-forms exists, it is unique, and it can be expressed by the following sum over eigenforms:
    \[
    \c{P}_t(x, x') = \sum_{n=0}^\infty e^{-\lambda_nt}\phi_n(x)\otimes\phi_n(x').
    \]
\end{theorem}

\begin{proof}
    See \textcite[Proposition~3.1]{rosenberg1997laplacian} for functions, the discussion at the end of \textcite[Section~3.2]{rosenberg1997laplacian} and \textcite{patodi71} for the general statement on forms.
\end{proof}

Given the heat kernel, the solution for the heat equation
\[
(\partial_t+\Delta_x)\alpha(t,x) = 0
\]
with initial condition $\alpha(0,x)=\beta(x)$ is given by
\[
\alpha(t,x)={}&\left(e^{-t\Delta}\beta\right)(x)\\
\coloneqq{}&\int_\M\left\langle \c{P}_t(x,x'),\beta(x')\right\rangle_{x'}\d x'\\
={}&\left\langle \c{P}_t(x,\cdot),\beta(\cdot)\right\rangle_{\c{L}^2(\M)}
\]
for $t>0$. We now show that the heat kernel is a valid kernel for Gaussian processes.

\begin{proposition}\label{prop: kernel symmetric and psd}
    The heat kernel satisfies the following properties.
    \begin{enumerate}
        \item\label{pt: symmetric} The kernel $\c{P}_t(x,x')$ is symmetric: for $x,x'\in\M$ we have $\c{P}_t(x,x')=\c{P}_t(x',x)$ under the canonical identification
        \[
        \Lambda^kT^*_x\M\otimes\Lambda^kT^*_{x'}\M\cong\Lambda^kT^*_{x'}\M\otimes\Lambda^kT^*_x\M.
        \]
        \item\label{pt: semigroup} The propagator $e^{-t\Delta}$ satisfies the semigroup property $e^{-(t+s)\Delta}=e^{-t\Delta}e^{-s\Delta}$.
        \item\label{pt: psd} The kernel $\c{P}_t(x,x')$ is positive semi-definite: for all $\alpha\in\Omega^k_{\c{L}^2}(\M)$ we have
        \[
        \left\langle\alpha(x),\left\langle \c{P}_t(x,x'),\alpha(x')\right\rangle_{x'}\right\rangle_x \geq 0.
        \]
    \end{enumerate}
\end{proposition}

\begin{proof}
    Statement \eqref{pt: symmetric} is immediate from \cref{thm: heat kernel as a sum}. Alternatively, it can be proved from the fact that the Hodge Laplacian is self-adjoint. Statement \eqref{pt: semigroup} follows in the same way as in \textcite[pp.~28--29]{rosenberg1997laplacian}, as does statement \eqref{pt: psd}.
\end{proof}

As an immediate corollary in the vector setting, we obtain the following.

\ThmExistenceGVFheatKernel*

\begin{proof}
    The statement follows from \textcite[Theorem~4]{hutchinson2021}, provided we prove that our kernel satisfies Definition~3 in the cited work. The link between the two definitions is given by
    \[
    k(\alpha_x,\beta_{x'}) = \left\langle\alpha_x,\left\langle \c{P}_t(x,x'),\beta_{x'}\right\rangle^{\text{pt}}_{x'}\right\rangle^{\text{pt}}_x,
    \]
    where the scalar products on the right-hand side are taken pointwise, not integrating over the manifold.

    Fibrewise bilinearity is obvious, we are left to prove that our kernel is positive semi-definite in the sense of \textcite[Definition~3]{hutchinson2021}. In order to do so, for $x_i\in\M$ and $\alpha_{x_i}\in T^*_{x_i}\M$, $1\le i\le n$, consider the sequence
    \[
    \alpha^m_i(x) = \left\langle\c{P}_{\frac{1}{m}}(x, x_i),\alpha_{x_i}\right\rangle^{\text{pt}}_{x_i}.
    \]
    Then we have
    \[
    \langle\c{P}_t&(x, x'),\alpha^m_i(x')\rangle_{x'}=\\
    &=\left\langle\c{P}_t(x, x'),\left\langle\c{P}_{\frac{1}{m}}(x', x_i),\alpha_{x_i}\right\rangle^{\text{pt}}_{x_i}\right\rangle_{x'}\\
    &=\left\langle\left\langle\c{P}_t(x, x'),\c{P}_{\frac{1}{m}}(x', x_i)\right\rangle_{x'},\alpha_{x_i}\right\rangle^{\text{pt}}_{x_i}\\
    &\xrightarrow{m\to\infty}\left\langle\c{P}_t(x, x_i),\alpha_{x_i}\right\rangle^{\text{pt}}_{x_i}
    \]
    and similarly
    \[
    \begin{aligned}
        k(\alpha_{x_i}&,\alpha_{x_j}) =\\
        &= \lim_{m\to\infty}\left\langle\alpha^m_i(x),\left\langle \c{P}_t(x,x'),\alpha^m_j(x')\right\rangle_{x'}\right\rangle_x.
    \end{aligned}
    \]
    Therefore, letting $\alpha^m(x) = \sum_{i=1}^n\alpha^m_i(x)$ we obtain
    \[
    \sum_{i=1}^n\sum_{j=1}^nk&(\alpha_{x_i},\alpha_{x_j})=\\
    ={}&\lim_{m\to\infty}\left\langle\alpha^m(x),\left\langle \c{P}_t(x,x'),\alpha^m(x')\right\rangle_{x'}\right\rangle_x\\
    \ge{}&0,
    \]
    since each term in the sequence is non-negative.
\end{proof}

\ThmSampling*

\begin{proof}
    The mean of the random variable $f(x)$ is obviously $0$. The covariance is easily computed to be
    \[
    \begin{aligned}
    \Cov(f(x)&,f(x'))=\\
    ={}&\frac{\sigma^2}{C_{\nu,\kappa}}\sum_{n=0}^L\Phi_{\nu,\kappa}(\lambda_n)s_n(x)\otimes s_n(x').
    \end{aligned}
    \]
    This concludes the proof.
\end{proof}

As a direct consequence of \Cref{thm: heat kernel as a sum}, \Cref{prop: kernel symmetric and psd}, and of the Hodge decomposition theorem, we also have the following.

\ThmDivCurlGVF*

There are some situations where it is only necessary to have partial information on the eigenforms of the Hodge Laplacian in order to gain full knowledge of the spectrum. We will now give the examples of surfaces---where it is enough to know the spectrum in the scalar case and the harmonic $1$-forms---and product manifolds---where we only need to know the spectrum on the factors separately.

\subsection{Surfaces}

Suppose now that $\dim\M=2$ and assume that we have an orthonormal basis $\{f_n\}_n$ of eigenfunctions of the Laplace--Beltrami operator on $\M$ with eigenvalues $\{-\lambda_n\}_n$ where $0=\lambda_1<\lambda_2\le\lambda_3\le\cdots$ (where we notice that the only eigenfunction with eigenvalue $0$ is the constant), as well as an orthonormal basis of $1$-eigenforms $\{\alpha_j\}_{0\le j\le J}$ of the first de Rham cohomology $H^1_{dR}(\M)$. Then, thanks to Hodge decomposition (\cref{hodge decomposition}) and \cref{proposition: transformations of eigenforms}, we have that an orthonormal basis of $1$-eigenforms of $\Omega^1(\M)$ is given by
\[\label{eq:surface basis}
\bigg\{\frac{\d f_n}{\sqrt{\lambda_n}}, \frac{\star\d f_n}{\sqrt{\lambda_n}}, \alpha_j\,\bigg|\,n \geq 1\text{ and }0\le j\le J\bigg\}.
\]
An orthonormal basis of $2$-forms is simply given by $\star f_i = f_i\omega$, where $\omega\in\Omega^2(\M)$ is the canonical volume form. Similarly, we can also recover all the $0$- and $2$-eigenforms from knowledge of the $1$-eigenforms.

In the dual case we have the following result from the main body of the paper.

\ThmBasisVFSurface*

\subsection{Products}\label{appdx:prod mflds}

In this section, let $\M,\NN$ be two oriented, connected, compact Riemannian manifolds. We recall the Stone--Weierstrass theorem, see e.g.~\textcite{prolla1994weierstrass}.

\begin{theorem}[Stone--Weierstrass]\label{stone-weierstrass}
    Let $A\subseteq C^\infty(\M)$ be a subalgebra of the algebra of smooth functions on $\M$ which contains all constant functions and which \emph{separates points}, i.e.~such that for each $x,y\in\M$ there is a function $f\in A$ with $f(x)\neq f(y)$. Then $A$ is dense in $C^\infty(\M)$.
\end{theorem}With this, we get the following, e.g.~\textcite[Section~4.6]{canzani13}.

\begin{proposition}\label{basis of eigenfunctions on product}
    For $X\in\{\M,\NN\}$, let $\{f^X_i\}_i\subseteq C^\infty(X)$ be an orthonormal basis of eigenfunctions of the Laplace--Beltrami operator on $X$ with eigenvalues~$-\lambda^X_i$. Then
    \[
    \bigg\{h_{i,j}(x_{\M},x_{\NN})\coloneqq f^\M_i(x_{\M})g^\NN_j(x_{\NN})\,\bigg|\,i,j\bigg\},
    \]
    where $(x_{\M},x_{\NN})\in\M\times\NN$, is an orthonormal basis of eigenfunctions of the Laplace--Beltrami operator on $\M\times\NN$ with eigenvalues $-(\lambda^\M_i + \lambda^\NN_j)$.
\end{proposition}

\begin{proof}
    It is straightforward to check that the functions $h_{i,j}$ are eigenfunctions with the prescribed eigenvalues, and that they are orthonormal. We need to show that these functions are dense in $\c{L}^2(\M\times\NN)$. The sets $\{f^X_i\}_i\subseteq C^\infty(X)$ are both dense in the spaces of smooth functions on the respective manifolds, so that the $h_{i,j}$ are dense in the subalgebra spanned by the products $C^\infty(\M)\cdot C^\infty(\NN)$. By the Stone--Weierstrass theorem, this subalgebra is dense in $C^\infty(\M\times\NN)$, and thus, also in $\c{L}^2(\M\times\NN)$, which concludes the proof.
\end{proof}

\begin{corollary}
    For the heat kernel on $\M\times\NN$ we have
    \[
    \c{P}^{\M\times\NN}_t(x,x') = \c{P}^\M_t(x_{\M},x_{\M}')\c{P}^\NN_t(x_{\NN},x_{\NN}'),
    \]
    for $x=(x_{\M},x_{\NN})\in \M\times\NN$ and similarly for $x'$.
\end{corollary}

The more general case of differential forms is similar. We need the following result, which will play the role of the Stone--Weierstrass theorem \cite{371618}. Recall that if $A$ is an algebra and $M$ is an $A$-module, an $A$-submodule of $M$ is a linear subspace $M'$ of $M$ such that for each $a\in A$ and $m'\in M$ we have $a\cdot m'\in M'$.

\begin{proposition}\label{stone-weierstrass for forms}
    Let $A\subseteq C^\infty(\M)$ be as in \cref{stone-weierstrass} and let $E\subseteq\Omega^k(\M)$ be an $A$-submodule which is such that for each point $x\in\M$ we have
    \[
    \mathrm{span}\{e(x)\mid e\in E\} = \Lambda^kT^*_x\M.
    \]
    Then $E$ is dense in $\Omega^k(\M)$ with respect to the supremum norm.
\end{proposition}

\begin{proof}
    First notice that we can choose a finite subset $e_1,\ldots,e_n\in E$ that spans $\Lambda^kT^*\M$ everywhere. Indeed, consider the collection of open sets given by
    \[
    \{x\in\M\mid e'_1(x),\ldots,e'_{n'}(x)\text{ spans }\Lambda^kT^*_x\M\}
    \]
    for all possible choices of $e'_1,\ldots,e'_{n'}\in E$. By our assumptions, this covers $\M$. Thus, by compactness we can choose a finite sub-collection that also covers $\M$. Taking all the elements of $E$ appearing in this finite sub-collection gives us the desired set. Every form $\alpha\in\Omega^k(\M)$ can then be written as
    \[
    \alpha(x)=\sum_{i=1}^nf_i(x)e_i(x)
    \]
    for some smooth $f_i$ (e.g.~using bump functions). Each of these functions $f_i$ can be approximated by elements of $A$, and since all sums of sections of the form $a\cdot e$ for $a\in A$ and $e\in E$ are in $E$ (as it is an $A$-module), it follows that we can approximate $\alpha$ as well.
\end{proof}

\begin{proposition}\label{basis of eigenforms on product}
    For $X\in\{\M,\NN\}$, let $\{\phi^{X,k}_i\}_i\subseteq\Omega^k(X)$ be an orthonormal basis of the eigenfields of the Hodge Laplacian with eigenvalues $\lambda^{X,k}_i$. The set
    \[
    \left\{\phi^{\M,k_\M}_i\wedge\phi^{\NN,k_\NN}_j\mid k_\M+k_\NN=k, i, j\right\}
    \]
    is an orthonormal basis of $\Omega^k(\M\times\NN)$ with eigenvalues $\lambda^{\M,k_\M}_i + \lambda^{\NN,k_\NN}_j$.
\end{proposition}

\begin{proof}
    Similar to \cref{basis of eigenfunctions on product}. Checking that these forms are eigenforms of the Hodge Laplacian with the prescribed eigenvalues and that they are orthonormal is straightforward. To show that their span is dense in $\Omega^k(\M\times\NN)$, we notice that for $x=(x_{\M},x_{\NN})\in\M\times\NN$ we have
    \[
    \Lambda^kT^*_x(\M\times\NN)\cong\bigoplus_{k_\M+k_\NN=k}T^*_{x_{\M}}\M\wedge T^*_{x_{\NN}}\NN.
    \]
    It follows that the span our set of forms can play the role of $E$ in \cref{stone-weierstrass for forms}, with the role of $A$ being played by the products of eigenfunctions
    \[
    A=\left\{\phi^{\M,0}_i\wedge\phi^{\NN,0}_j\mid i, j\right\}.
    \]
    The statement follows.
\end{proof}

\begin{corollary}\label{heat kernel on product manifolds}
    The heat kernel on $\Omega^k(\M\times\NN)$ is given by
    \[
    \begin{aligned}
    \c{P}^{\Omega^k(\M\times\NN)}_t&(x,x') =\\
    ={}&\sum_{i+j=k}\c{P}^{\Omega^{i}(\M)}_t(x_{\M},x_{\M}')\c{P}_t^{\Omega^{j}(\NN)}(x_{\NN},x_{\NN}')
    \end{aligned}
    \]
    for $x=(x_{\M},x_{\NN})\in\M\times\NN$ and similarly for $x'$.
\end{corollary}

A direct application of this is Hodge--Mat\'ern kernels on the torus, which we give in the context of vector fields.

\ThmMaternTorus*

\begin{proof}
    We have a canonical identification
    \[
    T\mathbb{S}_1\cong\mathbb{S}_1\times\R
    \]
    of the tangent bundle of the circle with a trivial bundle. It follows that we also have
    \[
    T\T^d\cong\T^d\times\R^d
    \]
    so that we can work in global coordinates as if we were working in $\R^d$.

    The Hodge star operator gives us an identification of vector fields on the circle with functions on the circle. By \Cref{basis of eigenforms on product} we obtain the desired form for the eigenfields on $\T^d$.

    The unnormalized form of the Hodge--Mat\'ern kernel is given by
    \[
    \begin{aligned}
    \sum_{j=1}^d&\sum_{\v{n}\in\N^d}e^{-t\lambda_{\v{n}}}\left(\prod_{i=1}^df_{n_i}(x_i)f_{n_i}(x'_i)\right)\v{e}_j\otimes\v{e}_j=\\
    &=\sum_{\v{n}\in\N^d}e^{-t\lambda_{\v{n}}}\left(\prod_{i=1}^df_{n_i}(x_i)f_{n_i}(x'_i)\right)\v{I}_d,
    \end{aligned}
    \]
    where $\v{n}=(n_1,\ldots,n_d)$ and $\lambda_{\v{n}}=\sum_{i=1}^d\lambda_{n_i}$. In the first part of the expression, we recognize the unnormalized form of the scalar manifold Mat\'ern kernel on $\T^d$. We then notice that the normalization factors of the scalar kernel and the vector one will differ by a factor $d$ by taking the trace of the identity matrix, which concludes the proof.
\end{proof}

\subsection{An Alternative: Connection Laplacian}\label{subsection:Bochner Laplacian}

There is another possible natural extension of the Laplace--Beltrami operator to differential forms, called the connection (or Bochner) Laplacian. In fact, it can be extended to much more than just differential forms: it exists as soon as we have a vector bundle on a manifold with a nicely behaved connection and inner product. An accessible introduction to this notion for the case of differential forms can be found in \textcite[Section~2.2]{rosenberg1997laplacian}, while an extensive treatise on it, including the existence of a heat kernel, in \textcite{berline1995heat}.

There is a formal relationship between the Hodge Laplacian and the connection Laplacian, called the Weitzenb\"ock formula, see \textcite[Section~2.2.2]{rosenberg1997laplacian} or \textcite[Section~13.12]{lee2009manifolds}: their difference is given by a term depending only on the curvature of the manifold. The Hodge theorems make it easier to work with the Hodge Laplacian by exploiting the relationships between the exterior derivative and the Hodge star operator, which is the reason why we focused on that operator in the present work. However, one gets valid kernels for GPs using the connection Laplacian as well, and this latter operator has some very nice properties that could be fruitful to use. For example, the heat kernel of the connection Laplacian asymptotically acts on vectors by parallel transport, see \textcite[Theorem~2.30]{berline1995heat}. This fact was exploited in \textcite{sharp2019vector} to provide efficient computational methods for parallel transport on manifolds.

\section{\bfseries\small FURTHER RESULTS}

This section contains the proofs of some additional results on Gaussian vector fields that were stated in the main body of the paper or that we found of general interest but could not include because of space limitations.

The first result formally shows that projected GPs built from stacking copies of known intrinsic scalar kernels will generally lead to undesirable uncertainty patterns. The second subsection studies normalization constants for the kernels of Hodge--Mat\'ern Gaussian vector fields. The third and last subsection aims to quantify the divergence of samples from different Gaussian vector fields.

\subsection{Limitations of Projected GPs}\label{sect:limitations of proj kernels}

We work on the standard unit sphere $\mathbb{S}_2\subset\R^3$. Let $\v{g} = \m{A} \v{h}$ where $h_j \sim \f{GP}(0, k_{\nu, \kappa_j, \sigma^2})$ for $i\in\{1,2,3\}$ are independent and $\m{A}\in\R^{3\times3}$ is an arbitrary matrix. Let $f$ be the associated projected vector GP.

\ThmLimitation*

\begin{proof}
For simplicity of notation, we assume the $\kappa_j$ are all equal. Otherwise, we can formally regard $\kappa_j \to \infty$ as $\min_j \kappa_j \to \infty$.

Take arbitrary $x, x' \in \mathbb{S}_2$. Then
\[
\Cov\del{f(x), f(x')}
&=
\Cov\del{\f{P}_x \v{g}(x), \f{P}_{x'}\v{g}(x')}
\\
&\!\!\!\!\!\!\!=
\f{P}_x \Cov\del{\v{g}(x), \v{g}(x')} \f{P}_{x'}^{\top}
\\
&\!\!\!\!\!\!\!=
\f{P}_x \m{A} \Cov\del{\v{h}(x), \v{h}(x')} \m{A}^{\top} \f{P}_{x'}^{\top}
.
\]
Since $\Cov\del{\v{h}(x), \v{h}(x')} = k_{\nu, \kappa, \sigma^2}(x, x') \m{I}$ and for arbitrary $x, x'$ we have $k_{\nu, \kappa, \sigma^2}(x, x') \to \sigma^2$ when $\kappa \to \infty$ \cite{borovitskiy2020}, we obtain
\[
\!\!\!\!\lim_{\kappa \to \infty}
\!\!\Cov\del{f(x), f(x')}
=
\sigma^2
\f{P}_x \m{A} \m{A}^{\top} \f{P}_{x'}^{\top}
=: \m{C}_{x x'}.
\]
Without loss of generality, we can assume $\sigma^2 = 1$. To analyze $\m{C}_{x x'}$ we construct the singular value decomposition (SVD) $\m{A} = \m{U} \m{\Sigma} \m{V}^{\top}$ of the matrix~$\m{A}$. Here $\m{U}$ and $\m{V}$ are orthogonal matrices and $\m{\Sigma}$ is a diagonal matrix with non-negative entries.
Then
\[
\!\m{A} \m{A}^{\top}
=
\m{U} \m{\Sigma} \underbracket{\m{V}^{\top} \m{V}}_{\m{I}} \m{\Sigma} \m{U}^{\top}
=
\m{U} \underbracket{\m{\Sigma}^2}_{=:\m{\Lambda}} \m{U}^{\top}
=
\m{U} \m{\Lambda} \m{U}^{\top}
\]
where the right-hand side is the eigendecomposition of the matrix~$\m{A} \m{A}^{\top}$.
We denote $\lambda_ i = \m{\Lambda}_{i i}$ and, without any loss of generality, assume that $\lambda_1 \leq \lambda_2 \leq \lambda_3$. By assumption, $\lambda_2,\lambda_3>0$.

The columns $\m{U}_{\cdot j}$ of $\m{U}$, $j=1, .., 3$, form an orthonormal basis of $\R^3$.
We choose $x = \m{U}_{\cdot 1}$ and $x' = \m{U}_{\cdot 2}$.
Then
\[
\f{P}_x
\m{U}_{\cdot j}
=
\f{P}_{\tilde{x}}
\m{U}_{\cdot j}
=
\begin{cases}
0 & \text{for } j = 1, \\
1 & \text{for } j = 2, 3.
\end{cases}
\\
\f{P}_{x'}
\m{U}_{\cdot j}
=
\begin{cases}
0 & \text{for } j = 2, \\
1 & \text{for } j = 1, 3.
\end{cases}
\]
It follows that
\[
\m{C}_{x x'}
&=
\del{\v{0}, \m{U}_{\cdot 2}, \m{U}_{\cdot 3}}
\,
\m{\Lambda}
\,
\del{\m{U}_{\cdot 1}, \v{0}, \m{U}_{\cdot 3}}^{\top}
\\
&=
\lambda_1 \v{0} \, \m{U}_{\cdot 1}^{\top}
+
\lambda_2 \m{U}_{\cdot 2} \, \v{0}^{\top}
+
\lambda_3 \m{U}_{\cdot 3} \, \m{U}_{\cdot 3}^{\top}
\\
&=
\lambda_3 \m{U}_{\cdot 3} \, \m{U}_{\cdot 3}^{\top}.
\]
Analogously, $C_{x \tilde{x}} = \lambda_2 \m{U}_{\cdot 2} \, \m{U}_{\cdot 2}^{\top} + \lambda_3 \m{U}_{\cdot 3} \, \m{U}_{\cdot 3}^{\top}$. Thus, we have
\[
\norm{\m{C}_{x x'}}_{F} = \sqrt{\tr \m{C}_{x x'} \m{C}_{x x'}^{\top}} = \lambda_3
\]
and similarly $\norm{\m{C}_{x \tilde{x}}}_{F} = \sqrt{\lambda_2^2 + \lambda_3^2}$, proving the claim.

\end{proof}

\subsection{Kernel Normalization Constants}

In \Cref{sec:hodge-matern}, we defined the normalization constant for the kernels of Hodge--Mat\'ern Gaussian vector fields implicitly by requiring that
\[\label{eq:normalization of vector kernel}
\frac{1}{\operatorname{vol}\M} \int_\M\tr\big(\tk(x,x)\big)\d x = \sigma^2.
\]
We will now explain what this normalization means in practice and make these constants explicit for the kernels of Hodge--Mat\'ern Gaussian vector fields and projected Mat\'ern Gaussian vector fields.

\begin{proposition}
    Assume \eqref{eq:normalization of vector kernel} holds, then
    \[
    \frac{1}{\operatorname{vol}\M}\mathbb{E}_{f\sim\GP(0,\v{k})}\left[\|f\|^2_{\c{L}^2(\M)}\right]=\sigma^2.
    \]
\end{proposition}

\begin{proof}
    We have
    \[
    \frac{1}{\operatorname{vol}\M}&\mathbb{E}_{f\sim\GP(0,\tk)}\left[\|f\|^2_{\c{L}^2(\M)}\right]=\\
    ={}&\frac{1}{\operatorname{vol}\M}\mathbb{E}_{f\sim\GP(0,\tk)} \sbr{\int_{\M} \|f(x)\|^2_2 \d x}\\
    ={}&\frac{1}{\operatorname{vol}\M}\int_{\M}\mathbb{E}_{f\sim\GP(0,\tk)}\left[\|f(x)\|^2_2\right] \d x\\
    ={}&\frac{1}{\operatorname{vol}\M}\int_{\M} \!\!\mathbb{E}_{f(x)\sim\c{N}(0,\tk(x,x))}\left[\|f(x)\|^2_2\right] \d x\\
    ={}&\frac{1}{\operatorname{vol}\M}\int_M\tr\big(\tk(x,x)\big) \d x\\
    ={}&\sigma^2,
    \]
    where the last equality holds by \eqref{eq:normalization of vector kernel}.
\end{proof}

\begin{proposition}\label{prop:kernel normalization constants}
    The constant $C_{\nu, \kappa}$ for the Hodge--Mat\'ern kernel $\v{k}_{\nu, \kappa,\sigma^2}$ is given by
    \[
    C_{\nu, \kappa} = \frac{1}{\operatorname{vol}\M}\sum_{n=0}^\infty\Phi_{\nu,\kappa}(\lambda_n),
    \]
    where the sum runs over all of the eigenfields of the Hodge Laplacian. Similar formulas are valid for the pure divergence, pure curl, and harmonic kernels by restricting the sum to the appearing eigenfields.
\end{proposition}

\begin{proof}
    We have
    \[
        \int_\M &\tr (\tk_{\nu, \kappa, \sigma^2}(x, x)) \d x =\\
        ={}&\int_\M\tr\left(\frac{\sigma^2}{C_{\nu, \kappa}}
\sum_{n=0}^{\infty}
\Phi_{\nu, \kappa}(\lambda_n)
s_n(x) \otimes s_n(x)\right)\d x\\
        ={}&\frac{\sigma^2}{C_{\nu, \kappa}}
\sum_{n=0}^{\infty}
\Phi_{\nu, \kappa}(\lambda_n)\int_\M\tr\left(s_n(x) \otimes s_n(x)\right)\d x\\
        ={}&\frac{\sigma^2}{C_{\nu, \kappa}}
\sum_{n=0}^{\infty}
\Phi_{\nu, \kappa}(\lambda_n)\int_\M\|s_n(x)\|_2^2\d x\\
        ={}&\frac{\sigma^2}{C_{\nu, \kappa}}
\sum_{n=0}^{\infty}
\Phi_{\nu, \kappa}(\lambda_n),
    \]
    since $\|s_n\|_{\c{L}^2(\M)}=1$. Notice that here the trace is taken with respect to the metric on $\M$. The result follows immediately by requiring the left-hand side of the equation to equal $\sigma^2\operatorname{vol}\M$.
\end{proof}

\begin{proposition}\label{prop:norm constant of proj}
    The appropriately normalized projected Mat\'ern kernel (with trivial coregionalization matrix) on a manifold $\M$ of dimension $d$ embedded in $\R^N$ is given by
    \[
    \v{k}^\pi_{\nu, \kappa, \sigma^2}&(x,x')=\frac{1}{d}k_{\nu, \kappa, \sigma^2}(x,x')P_x^T\v{I}_NP_{x'},
    \]
    where $k_{\nu, \kappa, \sigma^2}$ is the scalar manifold Mat\'ern kernel.
\end{proposition}

\begin{proof}
    Let $x=x'$ and pick a coordinate system such that $T_x\M=\operatorname{span}\{\v{e}_1,\ldots,\v{e}_d\}$. Then we have
    \[
    P_x^T\v{I}_NP_{x'}=\operatorname{diag}(\underbracket{1,\ldots,1}_{d},0,\ldots,0).
    \]
    It follows that
    \[
        \frac{1}{\operatorname{vol}\M}\int_\M\tr \Big(&\tk^P_{\nu, \kappa, \sigma^2}(x, x)\Big) \d x=
        \\
        ={}&\frac{1}{\operatorname{vol}\M} \int_\M k_{\nu, \kappa, \sigma^2}(x, x) \d x\\
        ={}&\sigma^2,
    \]
    as desired.
\end{proof}

\subsection{Divergence of Gaussian Vector Fields}\label{appdx:divergence of GPs}

We will now study the distribution of the (pointwise) divergence of the Hodge--Mat\'ern Gaussian vector fields and projected Mat\'ern GPs. Similar techniques can be used to compute the full distribution for $\d \alpha$ and $\d^\star\alpha$ for $\alpha\sim\GP(0,\v{k}_{\nu,\kappa,\sigma^2})$, which turns out to be another Gaussian process.

We fix a compact, oriented Riemannian manifold $\M$ of dimension $d_{\M}\ge1$ and we look at degree $k=1$ differential forms, although the computations straightforwardly generalize to all other $k$. We write $f_n\in C^\infty(\M)=\Omega^0(\M)$ for a basis of eigenfunctions with eigenvalues $-\lambda_n$, for $n\ge0$, where $0=\lambda_0<\lambda_1\le\cdots$. We also set $\phi_n,n\in\N$ a basis of $1$-eigenforms with eigenvalues including the eigenforms
\[\label{eq:phi=df}
\frac{1}{\sqrt{\lambda_n}}\d f_n
\]
for $n\ge1$. We use $C^k_{\nu,\kappa}$ to denote the normalization constant for the Hodge--Mat\'ern kernel on $k$-forms (with $k=0$ being the case of functions).
For $\alpha\sim\GP(0,\tk)$ a Gaussian differential form, we write $\div\alpha(x)$ for the random variable $\d^\star\alpha(x)$, where $x\in\M$.
We assume $\div\alpha(x)$ is well-defined, i.e. the Gaussian process is smooth enough, which places restrictions on the parameter $\nu$.
We leave the precise nature of these out of the scope of this paper.

\begin{proposition}
    For the Hodge--Mat\'ern Gaussian form $\alpha_{\nu, \kappa, \sigma} \sim \GP(0, \v{k}_{\nu,\kappa,\sigma^2})$ on $1$-forms we have
    \[
    \!\!\!\!\Var\del{\div \alpha_{\nu, \kappa, \sigma}(x)} \!=\! \frac{\sigma^2}{C^1_{\nu,\kappa}}\sum_{n=1}^\infty\lambda_n\Phi_{\nu,\kappa}(\lambda_n)f_n(x)^2
    \]
    whenever $\alpha_{\nu, \kappa, \sigma}$ is smooth enough for the divergence to be well-defined.
\end{proposition}

\begin{proof}
    By \Cref{thm:sampling} we have
    \[
    \div &\alpha_{\nu, \kappa, \sigma}(x)
    =
    \d^\star \alpha_{\nu, \kappa, \sigma}(x)
    \\
    &= \d^\star\left(\frac{\sigma}{\sqrt{C^1_{\nu,\kappa}}}\sum_{n=0}^\infty \sqrt{\Phi_{\nu,\kappa}(\lambda_n)}w_n\phi_n(x)\right)\\
    &=\frac{\sigma}{\sqrt{C^1_{\nu,\kappa}}}\sum_{n=0}^\infty \sqrt{\Phi_{\nu,\kappa}(\lambda_n)}w_n\d^\star\phi_n(x)\\
    &=\frac{\sigma}{\sqrt{C^1_{\nu,\kappa}}}\sum_{n=1}^\infty \sqrt{\frac{\Phi_{\nu,\kappa}(\lambda_n)}{\lambda_n}}w_n\d^\star\d f_n(x)\\\label{eqn:can_fail_to_converge}
    &=\frac{\sigma}{\sqrt{C^1_{\nu,\kappa}}}\sum_{n=1}^\infty \sqrt{\lambda_n\Phi_{\nu,\kappa}(\lambda_n)}w_n f_n(x),
    \]
    where in the third line we used that $\d^\star\phi_n(x)=0$ unless $\phi_n(x)$ is of the form \eqref{eq:phi=df} and in the last line we used the fact that $\d^\star\d f_n(x)=-\Delta f_n(x)=\lambda_nf_n(x)$. This gives the full distribution for $\div \alpha_{\nu, \kappa, \sigma}(x)$, and in particular the desired formula for the variance.
    Note that~\Cref{eqn:can_fail_to_converge} can fail to converge when $\nu$ is not large enough (i.e. $\alpha_{\nu, \kappa, \sigma}$ is not smooth enough). 
\end{proof}

\begin{corollary}\label{cor:HM div on sphere}
    On the sphere, in terms of vector fields, we have
    \[
    \Var\del{\div f_{\nu, \kappa, \sigma}(x)}=\frac{\sigma^2}{2}\frac{\sum_{n=1}^\infty \lambda_n\Phi_{\nu,\kappa}(\lambda_n)}{4\pi \big(C^0_{\nu,\kappa}-\Phi_{\nu,\kappa}(0)\big)}
    \]
    for any $x\in \mathbb{S}_2$.
\end{corollary}

\begin{proof}
    By symmetry, $\Var\del{\div f_{\nu, \kappa, \sigma}(x)}$ does not depend on $x$, i.e.~it is a constant. Therefore, we have
    \[
    4\pi\Var(&\div f_{\nu, \kappa, \sigma}(x))= \\
    &=\int_{\M}\Var\del{\div f_{\nu, \kappa, \sigma}(x')}\d x'\\
    &=\int_{\M}\frac{\sigma^2}{C^1_{\nu,\kappa}}\sum_{n=1}^\infty\lambda_n\Phi_{\nu,\kappa}(\lambda_n)f_n(x')^2\d x'\\
    &=\frac{\sigma^2}{C^1_{\nu,\kappa}}\sum_{n=1}^\infty\lambda_n\Phi_{\nu,\kappa}(\lambda_n)
    \]
    since $\|f_n\|_{\c{L}^2(\M)}=1$. The statement follows by noticing that each eigenvalue in the spectrum on functions appears twice in the spectrum on vector fields---except for $0$, which does not appear.
\end{proof}

We now look at the projected kernel. Suppose $\phi:\M\to\R^N$ is an isometric embedding and write $\v{k}^\pi_{\nu,\kappa,\sigma^2}$ for the projected kernel obtained by projecting the vector kernel in $\R^N$ where each component is a scalar manifold Mat\'ern kernel with the same hyperparameters $\nu,\kappa,\sigma^2$. In this case, we will talk about vectors and gradients instead of differential forms (although a formulation in terms of $1$-forms is also possible).

We write $P:\phi^*(T\R^N)\to T\M$ for the orthogonal projection to the tangent bundle of $\M$. The proof of the following helpful lemma was provided by Alan Pinoy in a private communication.

\begin{lemma}\label{lemma:div proj}
    Let $f:\M\to\R$ be a smooth function and let $\v{w}\in\R^N$ be a fixed vector. Then
    \[
    \div(f(x)P_x\v{w})=\big(\nabla f(x) + f(x) H(x)\big)\cdot\v{w},
    \]
    where $H$ is the mean curvature vector of the embedding defined in~\Cref{eqn:mean_curv_vec}.
\end{lemma}

\begin{proof}
    We denote by $\nabla$ the Levi--Civita connection of $\M$ and by $\overline{\nabla}$ that of $\R^N$. Recall that the (vector) \emph{second fundamental form} of the embedding is given by
    \[
    \ssf(u,v)=\overline{\nabla}_uv-\nabla_uv\in\phi^*(T\R^N),
    \]
    cf.~\textcite[Section~4.2]{lee2009manifolds} (for the scalar version). Intuitively, it measures the infinitesimal curvature of $\M$ inside of $\R^N$. Notice that it is always orthogonal to the tangent space of $\M$. The \emph{mean curvature vector} is the trace of the second fundamental form: if $e_1,\ldots,e_{d_{\M}}$ is a local orthonormal frame of $\M$, then
    \[ \label{eqn:mean_curv_vec}
    H(x)=\sum_{i=1}^{d_{\M}}\ssf(e_i,e_i).
    \]
    With these definitions at hand, we have
    \[\label{eq:grad + div Pw}
    \div(f(x)P_x\v{w})=\nabla f(x)\cdot\v{w} + f(x)\div(P_x\v{w}).
    \]
    We use the standard differential geometric notation $v(f)$ for the derivative of $f$ in direction $v$, where $f$ is a function and $v$ is a vector field. For this last term, with $g$ denoting the metric on $\M$, we compute
    \[
    \div(P_x&\v{w})=\tr\big(\nabla(P_x\v{w})\big)\\
    ={}&\sum_{i=1}^{d_{\M}}g(\nabla_{e_i}(P_x\v{w}),e_i)\\
    ={}&\sum_{i=1}^{d_{\M}}e_i\big(g(P_x\v{w},e_i)\big) - g(P_x\v{w},\nabla_{e_i}e_i)\label{eq:used metric}\\
    ={}&\sum_{i=1}^{d_{\M}}e_i(P_x\v{w}\cdot e_i) - P_x\v{w}\cdot\nabla_{e_i}e_i\label{eq:used isometric}\\
    ={}&\sum_{i=1}^{d_{\M}}e_i(\v{w}\cdot e_i\big) - \v{w}\cdot\nabla_{e_i}e_i\\
    ={}&\sum_{i=1}^{d_{\M}}\overline{\nabla}_{e_i}\v{w}\cdot e_i + \v{w}\cdot \overline{\nabla}_{e_i}e_i - \v{w}\cdot\nabla_{e_i}e_i\label{eq:used babla bar levi-civita}\\
    ={}&\sum_{i=1}^{d_{\M}}\v{w}\cdot\left(\overline{\nabla}_{e_i}e_i - \nabla_{e_i}e_i\right)\\
    ={}&\v{w}\cdot\left(\sum_{i=1}^{d_{\M}}\ssf(e_i,e_i)\right)\\
    ={}&\v{w}\cdot H(x),
    \]
    where \eqref{eq:used metric} comes from the fact that $\nabla$ is the Levi--Civita connection, \eqref{eq:used isometric} from the fact that $\phi$ is an isometric embedding, and \eqref{eq:used babla bar levi-civita} comes from the fact that $\overline{\nabla}$ is the Levi--Civita connection of $\R^N$ (the metric being given by the dot product). Inserting this in \eqref{eq:grad + div Pw} concludes the proof.
\end{proof}

\begin{proposition}
    For the projected Mat\'ern GP $f^\pi_{\nu,\kappa,\sigma^2} \sim \GP(0, \v{k}^\pi_{\nu,\kappa,\sigma^2})$ we have
    \[
    \Var(&\div f^\pi_{\nu,\kappa,\sigma^2})=
    \\
    ={}&\frac{\sigma^2}{d_{\M}C^0_{\nu,\kappa}}\Bigg(\sum_{n=1}^\infty \lambda_n \Phi_{\nu,\kappa}(\lambda_n)\left\|\frac{\nabla f_n(x)}{\sqrt{\lambda_n}}\right\|_2^2+\\
    &+\sum_{n=0}^\infty \Phi_{\nu,\kappa}(\lambda_n)f_n(x)^2\left\|H(x)\right\|_2^2\Bigg)
    \]
    whenever $f^\pi_{\nu,\kappa,\sigma^2}$ is smooth enough for the divergence to be well-defined.
\end{proposition}

\begin{proof}
    Once again, by \Cref{thm:sampling} we have
    \[
    \begin{aligned}
    f^\pi_{\nu,\kappa,\sigma^2}&(x)=\\
    ={}&\frac{\sigma}{\sqrt{d_{\M}C^0_{\nu,\kappa}}}\sum_{n=0}^\infty\sqrt{\Phi_{\nu,\kappa}(\lambda_n)}f_n(x)P_x\v{w}_n
    \end{aligned}
    \]
    where $\v{w}_n\sim\c{N}(0,\v{I}_N)$ is a sequence of i.i.d.~multivariate normal vectors. Taking the divergence by applying \Cref{lemma:div proj} to each summand, we obtain
    \[
    \begin{aligned}
    \div&f^\pi_{\nu,\kappa,\sigma^2}(x)=\\
    ={}&\frac{\sigma}{\sqrt{d_{\M}C^0_{\nu,\kappa}}}\sum_{n=1}^\infty\sqrt{\Phi_{\nu,\kappa}(\lambda_n)}\nabla f_n(x)\cdot\v{w}_n+\\
    +{}&\frac{\sigma}{\sqrt{d_{\M}C^0_{\nu,\kappa}}}\sum_{n=0}^\infty\sqrt{\Phi_{\nu,\kappa}(\lambda_n)}f_n(x)H(x)\cdot\v{w}_n.
    \end{aligned}
    \]
    Since the $\v{w}_n$ are independent and
    \[
    \begin{aligned}
    \|\nabla f_n(x) +{}& f_n(x)H(x)\|^2=\\
    ={}&\|\nabla f_n(x)\|^2+f_n(x)^2\|f_n(x)H(x)\|^2
    \end{aligned}
    \]
    as $\nabla f_n(x)$ and $H(x)$ are orthogonal, the statement follows.
\end{proof}

\begin{corollary}\label{cor:proj M div on sphere}
    On the sphere we have
    \[
    \begin{aligned}
    \Var&(\div f^\pi_{\nu,\kappa,\sigma^2})=\\
    &=\frac{\sigma^2}{2}\left(\frac{\sum_{n=1}^\infty \lambda_n\Phi_{\nu,\kappa}(\lambda_n)}{4\pi C^0_{\nu,\kappa}} + 1\right)
    \end{aligned}
    \]
    for any $x\in\mathbb{S}_2$.
\end{corollary}

\begin{proof}
    By the same symmetry argument as in \Cref{cor:HM div on sphere} and using the fact that for the standard unit sphere embedding we have $H(x)=\v{n}(x)$ the unit normal vector, we obtain
    \[
    \begin{aligned}
        4\pi&\Var(\div f^\pi_{\nu,\kappa,\sigma^2})=\\
        &=\frac{\sigma^2}{2C^0_{\nu,\kappa}}\left(\sum_{n=1}^\infty\lambda_n\Phi_{\nu,\kappa}(\lambda_n) + 4\pi C^0_{\nu,\kappa}\right),
    \end{aligned}
    \]
    which is the stated result.
\end{proof}

\section{\bfseries\small ADDITIONAL EXPERIMENTAL DETAILS} \label{appendix:experimental details}

\subsection{Weather Modeling}

In the weather modeling experiments, the signal was normalized before fitting by scaling it by a constant scalar value so that the mean norm of the training observations is $1$.

The training points were selected by taking the observations at longitudes $90^\circ$E and $90^\circ$W, and then picking one every $180$ of them, spaced regularly. This gives a final training set of $34$ points. The $1220$ testing points were generated randomly.\footnote{Specifically, we used the Poisson disk sampling routine \texttt{create\_poisson\_disk\_samples} on the sphere from \href{https://github.com/aterenin/phdthesis/blob/d1476cbdd9c92b788505b25077b38b92a35c35ca/render/render.py\#L733}{https://github.com/aterenin/phdthesis}.}

Using the results of \Cref{appdx:divergence of GPs}, it was possible to quantify the variance of the divergence of the Gaussian vector fields arising from the (prior) kernels that were fitted in the experiments. The results are displayed in \Cref{fig:var div quantification}, confirming that the absolute divergence was higher for the Hodge heat and Hodge--Mat\'ern kernels, which explains the worse performance of these intrinsic kernels against the projected kernels for this experiment.

An examination of the fitted hyperparameters revealed that the typical length scale for the divergence-free Hodge--Mat\'ern kernels is between $\kappa=0.1$ and $\kappa=0.3$. In some rare exceptions, hyperparameter fitting converged to a local optimum with large length scale, but performance was not overly affected. After visual inspection of the ground truth, we fitted divergence-free Hodge--Mat\'ern kernels with fixed length scales of $\kappa=0.5$ and $\kappa=1$. The results are reported in \Cref{table:results ERA5 orbit fixed kappa}. We notice potential small improvements in performance, but not statistically significant ones.

In the last experiment on this data, we fitted the Hodge-compositional Mat\'ern kernels, i.e. linear combinations of pure-divergence and pure-curl Hodge--Mat\'ern kernels, as reported also in \Cref{table:results ERA5 orbit}. We see that these kernels recover almost exactly the results of fitting a divergence-free Hodge--Mat\'ern kernel. A detailed analysis of the fitted hyperparameters shows that the resulting length scales and variances put full weight on the divergence-free part of the kernel with an almost exact match of length scales. The only exceptions are when the length scale converges to a local optimum, which fully explains the minimal advantage in performance of the linear combination kernel over the divergence-free Hodge--Mat\'ern kernel with $\nu=\tfrac{1}{2}$. This supports the fact that Hodge-compositional Mat\'ern kernels are able to automatically recover appropriate inductive biases in such situations, mirroring their discrete counterparts in~\textcite{yang2024}.

\subsection{Synthetic Experiments}

\begin{table*}[t!]%
\centering
\begin{tabular}{lrrrrrrrrrr}
\toprule
\multirow{2}{*}{\textbf{Kernel}} & \multicolumn{2}{r}{H.--M.--$\tfrac{1}{2}$ sample} & \multicolumn{2}{r}{H.--M.--$\infty$ sample} & \multicolumn{2}{r}{P.~M.--$\tfrac{1}{2}$ sample} & \multicolumn{2}{r}{Rotation field} & \multicolumn{2}{r}{\begin{tabular}{@{}c@{}}curl-free \\ H.--M.--$\tfrac{1}{2}$ sample\end{tabular}} \\
 \cmidrule(lr){2-3}\cmidrule(lr){4-5}\cmidrule(lr){6-7}\cmidrule(lr){8-9}\cmidrule(lr){10-11}& Mean & Std & Mean & Std & Mean & Std & Mean & Std & Mean & Std \\
\midrule
Pure noise & 0.17 & 0.04 & 1.18 & 0.32 & 0.22 & 0.06 & 0.68 & 0.02 & 0.08 & 0.02 \\
P.~M.--$\tfrac{1}{2}$ & 0.14 & 0.03 & 0.87 & 0.42 & \bf 0.16 & 0.05 & 0.06 & 0.02 & 0.07 & 0.02 \\
P.~M.--$\infty$ & 0.19 & 0.04 & 0.71 & 0.26 & 0.20 & 0.07 & 0.34 & 0.36 & 0.07 & 0.03 \\
H.--M.--$\tfrac{1}{2}$ & \bf 0.14 & 0.03 & 0.84 & 0.38 & 0.16 & 0.05 & 0.02 & 0.01 & 0.07 & 0.02 \\
H.--M.--$\infty$ & 0.17 & 0.04 & \bf 0.65 & 0.25 & 0.20 & 0.06 & 0.00 & 0.00 & 0.08 & 0.02 \\[0.5ex]
\begin{tabular}{@{}l@{}}curl-free \\ H.--M.--$\tfrac{1}{2}$\end{tabular} & 0.20 & 0.06 & 1.15 & 0.47 & 0.22 & 0.07 & 0.68 & 0.02 & \bf 0.05 & 0.01 \\[2.0ex]
\begin{tabular}{@{}l@{}}curl-free \\ H.--M.--$\infty$\end{tabular} & 0.16 & 0.04 & 1.11 & 0.37 & 0.23 & 0.06 & 0.69 & 0.03 & 0.08 & 0.02 \\[2.0ex]
\begin{tabular}{@{}l@{}}div-free \\ H.--M.--$\tfrac{1}{2}$\end{tabular} & 0.15 & 0.05 & 1.00 & 0.60 & 0.19 & 0.05 & 0.01 & 0.00 & 0.08 & 0.02 \\[2.0ex]
\begin{tabular}{@{}l@{}}div-free \\ H.--M.--$\infty$\end{tabular} & 0.16 & 0.04 & 0.81 & 0.46 & 0.18 & 0.04 & \bf 0.00 & 0.00 & 0.08 & 0.02 \\
\bottomrule
\end{tabular}

\caption{MSE for synthetic experiments. The columns are datasets, the rows are models.}%
\label{table:synthetic MSE}%
\end{table*}

\begin{table*}[t!]%
\centering
\begin{tabular}{lrrrrrrrrrr}
\toprule
\multirow{2}{*}{\textbf{Kernel}} & \multicolumn{2}{r}{H.--M.--$\tfrac{1}{2}$ sample} & \multicolumn{2}{r}{H.--M.--$\infty$ sample} & \multicolumn{2}{r}{P.~M.--$\tfrac{1}{2}$ sample} & \multicolumn{2}{r}{Rotation field} & \multicolumn{2}{r}{\begin{tabular}{@{}c@{}}curl-free \\ H.--M.--$\tfrac{1}{2}$ sample\end{tabular}} \\
 \cmidrule(lr){2-3}\cmidrule(lr){4-5}\cmidrule(lr){6-7}\cmidrule(lr){8-9}\cmidrule(lr){10-11}& Mean & Std & Mean & Std & Mean & Std & Mean & Std & Mean & Std \\
\midrule
Pure noise & 0.41 & 0.25 & 2.39 & 0.34 & 0.67 & 0.29 & 1.76 & 0.04 & -0.31 & 0.25 \\
P.~M.--$\tfrac{1}{2}$ & 0.15 & 0.26 & 2.18 & 0.78 & \bf 0.31 & 0.27 & -0.65 & 0.20 & -0.51 & 0.34 \\
P.~M.--$\infty$ & 0.61 & 0.39 & 1.47 & 0.63 & 0.58 & 0.40 & -3.43 & 5.47 & -0.38 & 0.73 \\
H.--M.--$\tfrac{1}{2}$ & \bf 0.13 & 0.24 & 2.12 & 0.71 & 0.33 & 0.28 & -1.41 & 0.15 & -0.58 & 0.28 \\
H.--M.--$\infty$ & 0.41 & 0.25 & \bf 1.27 & 0.48 & 0.53 & 0.31 & -9.42 & 0.04 & -0.31 & 0.25 \\[0.5ex]
\begin{tabular}{@{}l@{}}curl-free \\ H.--M.--$\tfrac{1}{2}$\end{tabular} & 0.66 & 0.59 & 2.66 & 0.71 & 0.67 & 0.37 & 1.77 & 0.03 & \bf -0.79 & 0.21 \\[2.0ex]
\begin{tabular}{@{}l@{}}curl-free \\ H.--M.--$\infty$\end{tabular} & 0.38 & 0.30 & 2.55 & 0.58 & 0.73 & 0.32 & 1.77 & 0.05 & -0.31 & 0.25 \\[2.0ex]
\begin{tabular}{@{}l@{}}div-free \\ H.--M.--$\tfrac{1}{2}$\end{tabular} & 0.25 & 0.31 & 2.37 & 1.16 & 0.48 & 0.31 & -2.20 & 0.17 & -0.37 & 0.31 \\[2.0ex]
\begin{tabular}{@{}l@{}}div-free \\ H.--M.--$\infty$\end{tabular} & 0.33 & 0.25 & 1.52 & 0.66 & 0.46 & 0.18 & \bf -9.63 & 0.00 & -0.31 & 0.25 \\
\bottomrule
\end{tabular}

\caption{Predictive NLL for synthetic experiments. The columns are datasets, the rows are models.}%
\label{table:synthetic PNLL}%
\end{table*}

In addition to the weather modeling experiment detailed in the main body of the paper, we ran various experiments on synthetically generated data.

Namely, we considered samples drawn from each of the following GPs on the sphere: Hodge--Mat\'ern ($\nu=\infty,\kappa=0.5$), projected Mat\'ern ($\nu=\tfrac{1}{2},\kappa=0.5$), Hodge--Mat\'ern ($\nu=\tfrac{1}{2},\kappa=0.5$), and curl-free Hodge--Mat\'ern ($\nu=\tfrac{1}{2},\kappa=0.5$). In addition to this, we also used the "rotation" vector field of \Cref{fig:rotation vf}, which is given by
\[
\begin{pmatrix}x\\y\\z\end{pmatrix} \longmapsto \begin{pmatrix}y\\-x\\0\end{pmatrix}.
\]
This vector field is pure curl, as it can obtained as the curl of the function $f(x,y,z)=z$ on the sphere. For each experiment, $30$ training points were selected uniformly at random from the northern hemisphere, and $100$ testing points were selected---also uniformly at random---from the southern hemisphere. Each experiment was repeated $10$ times: for the rotation vector field the training and testing points were resampled at each experiment, while for the others a new sample was drawn for each experiment.

In each experiment, we fitted the following GPs: pure noise, projected Mat\'ern (P.~M.), Hodge--Mat\'ern (H.--M.), divergence-free Hodge--Mat\'ern (div-free H.--M.), and curl-free Hodge--Mat\'ern (curl-free H.--M.), all with $\nu=\tfrac{1}{2},\infty$.
All kernels also fitted variance and an additive noise.

The results are reported in \Cref{table:synthetic MSE,table:synthetic PNLL}. We see that for each experiment based on samples, the respective kernel performed best. On the rotation vector field, Hodge--Mat\'ern and divergence-free Hodge--Mat\'ern vastly outperformed all other kernels.

\newpage

\begin{figure*}[p]
\centering
\begin{subfigure}{0.75\textwidth}%
\includegraphics[width=\textwidth]{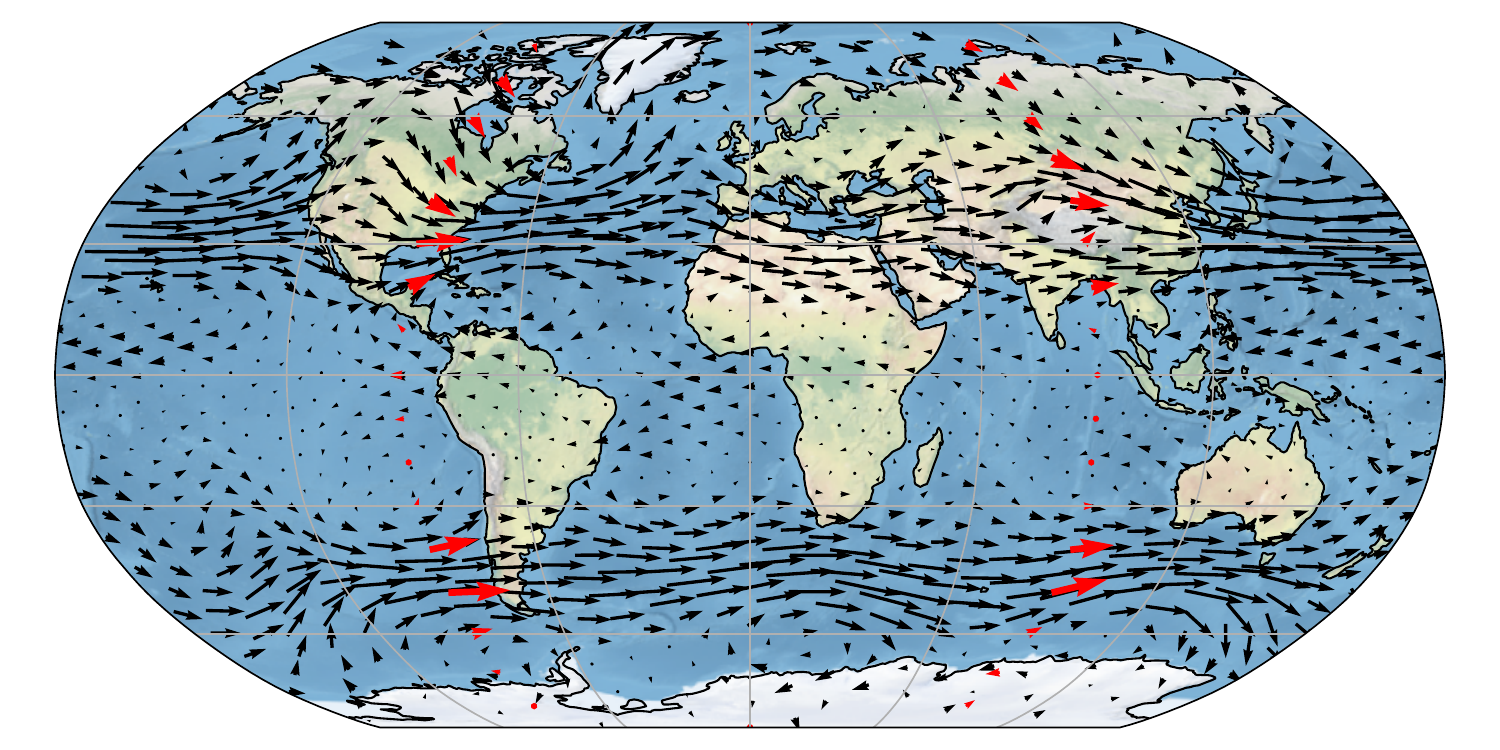}%
\caption{Ground truth (January 2010) and observations.}%
\end{subfigure}%
\\
\begin{subfigure}{0.75\textwidth}%
\includegraphics[width=\textwidth]{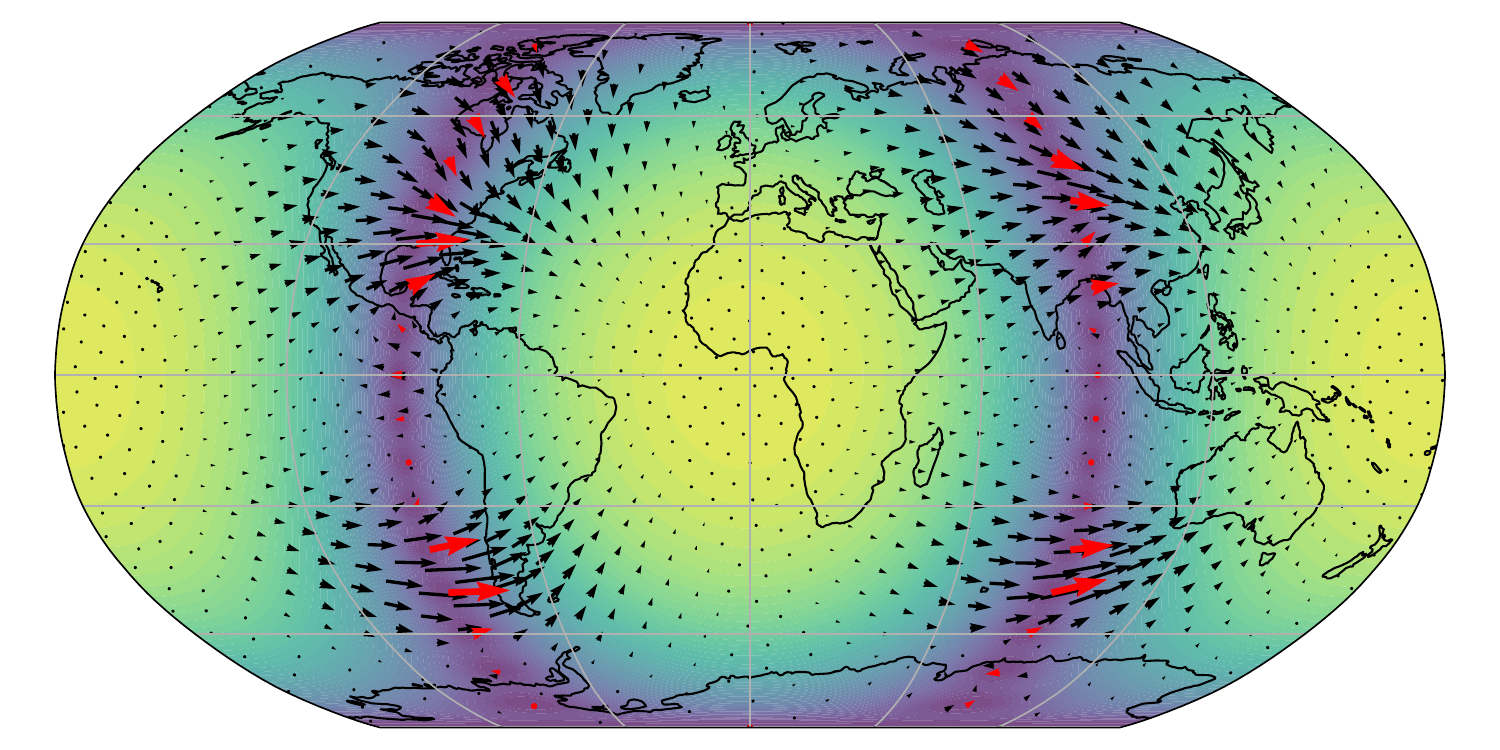}%
\caption{Predictive mean and uncertainty (blue is low and yellow is high).}%
\end{subfigure}%
\\
\begin{subfigure}{0.75\textwidth}%
\includegraphics[width=\textwidth]{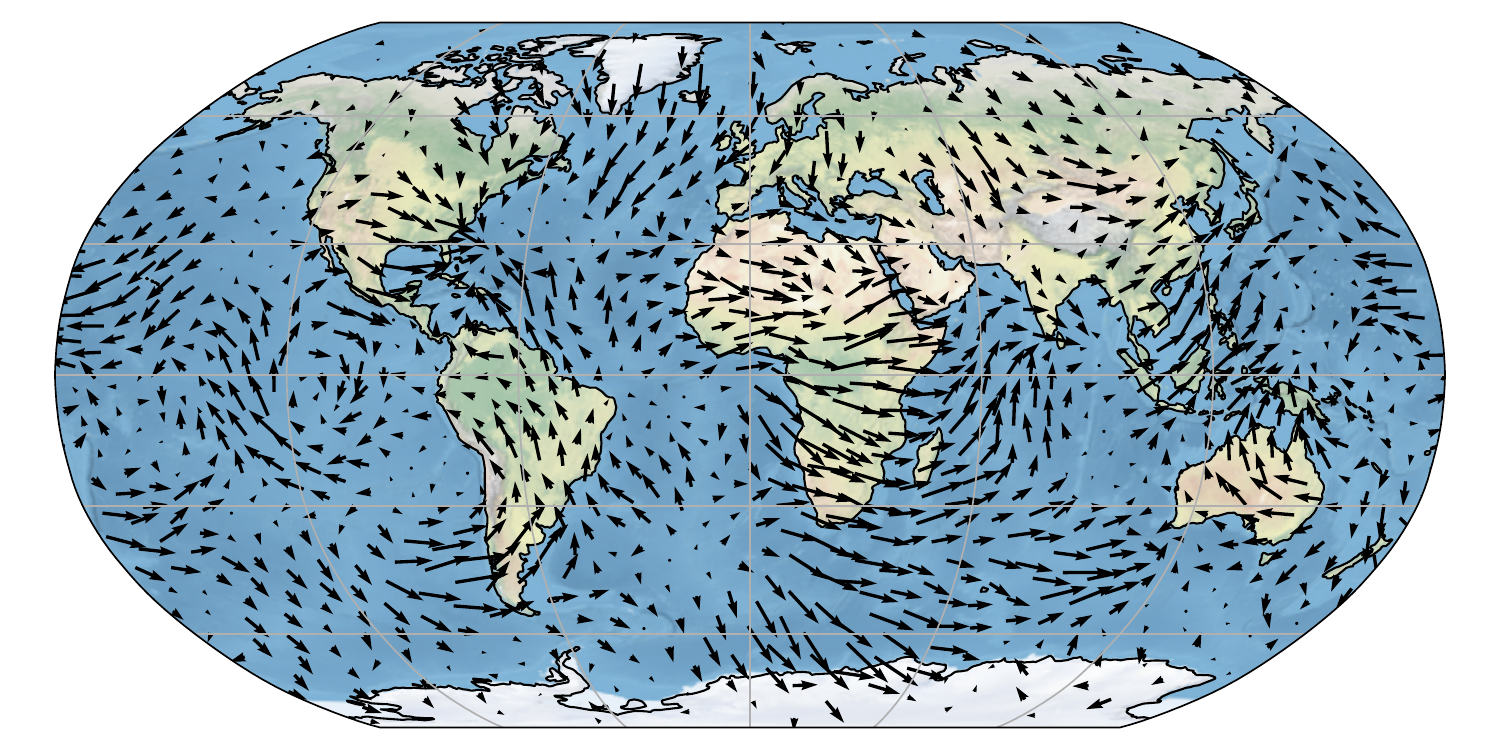}%
\caption{Posterior sample.}%
\end{subfigure}%
\caption{Robinson projection of \Cref{fig:winds ground truth,fig:winds proj mean,fig:winds proj sample} displaying the ground truth, observations, predictive mean, uncertainty, and a posterior sample of the GP with projected Mat\'ern kernel with $\nu=\tfrac{1}{2}$ and length scale $\kappa=0.5$. The vectors in the sample are scaled independently from the ground truth and predictive mean. Visually, the sample does not have structures reminiscent of that of the ground truth.}
\label{fig:flat wind modeling proj}
\end{figure*}

\begin{figure*}[p]
\centering
\begin{subfigure}{0.75\textwidth}%
\includegraphics[width=\textwidth]{figures/fig89a.pdf}%
\caption{Ground truth (January 2010) and observations.}%
\end{subfigure}%
\\
\begin{subfigure}{0.75\textwidth}%
\includegraphics[width=\textwidth]{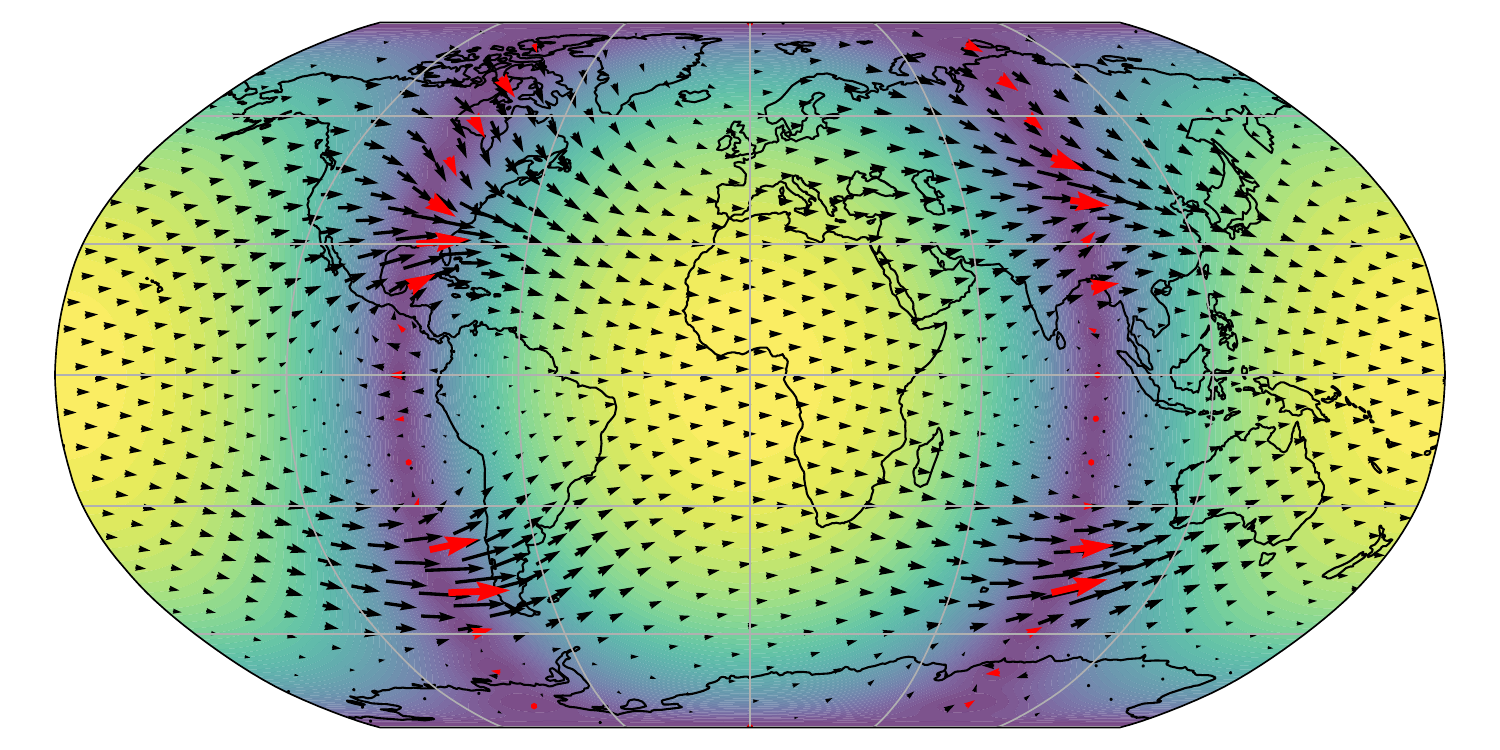}%
\caption{Predictive mean and uncertainty (blue is low and yellow is high).}%
\end{subfigure}%
\\
\begin{subfigure}{0.75\textwidth}%
\includegraphics[width=\textwidth]{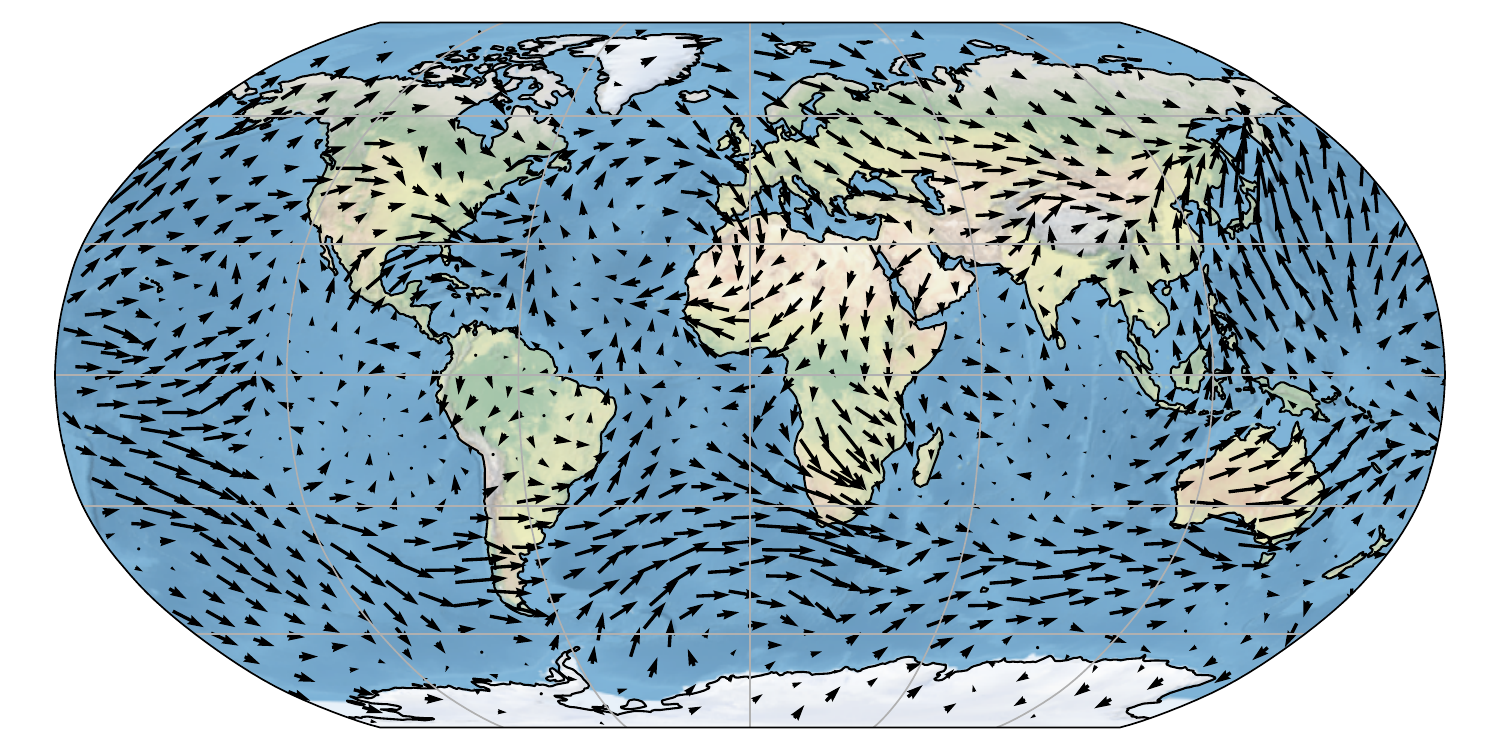}%
\caption{Posterior sample.}%
\end{subfigure}%
\caption{Robinson projection of \Cref{fig:winds ground truth,fig:winds hodge mean,fig:winds hodge sample} displaying the ground truth, observations, predictive mean, uncertainty, and a posterior sample of the GP with divergence-free Hodge--Mat\'ern kernel with $\nu=\tfrac{1}{2}$ and length scale $\kappa=0.5$. The vectors in the sample are scaled independently from the ground truth and predictive mean. Visually, and to the contrary of \Cref{fig:flat wind modeling proj}, the sample appears to have structures reminiscent of that of the ground truth, such as a strong west to east current in the southern hemisphere.}
\label{fig:flat wind modeling hodge}
\end{figure*}

\begin{figure*}[t]
\centering
\includegraphics[width=\textwidth]{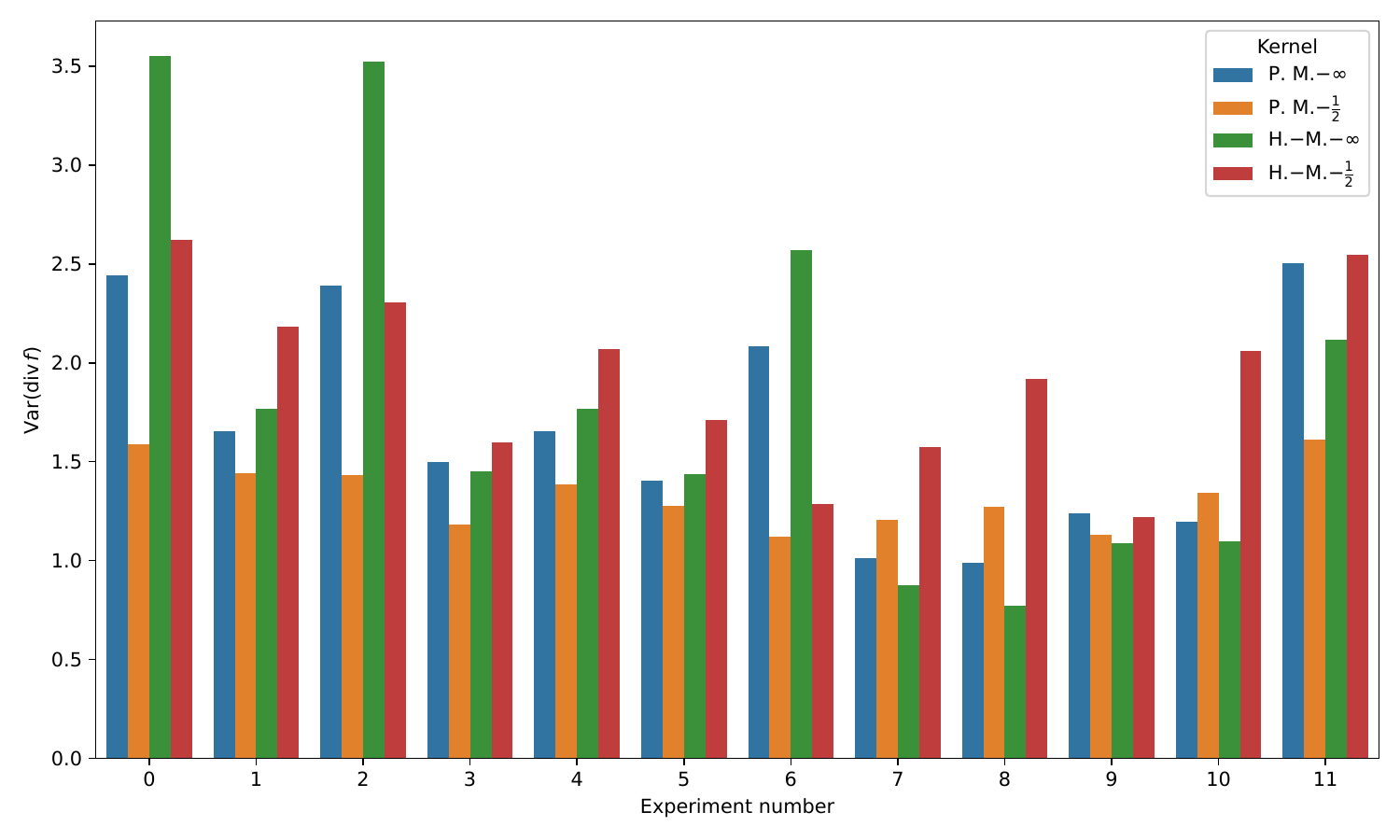}%
\caption{Variance of divergence for the prior kernels with fitted hyperparameters in the weather modeling experiments. Note that the variance of the divergence does not depend on the input location in this case. The divergence-free kernels and the kernels with fixed length scale are not represented here.}%
\label{fig:var div quantification}
\end{figure*}

\end{document}